\documentclass{podraft}

\usepackage[english]{babel}
\usepackage[utf8]{inputenc}
\usepackage[T1]{fontenc}
\usepackage{mathtools}
\usepackage{amsfonts}
\usepackage{amssymb}
\usepackage{microtype}
\usepackage{graphicx}
\usepackage{booktabs}
\usepackage{wrapfig}
\usepackage{subcaption}
\usepackage{placeins}
\usepackage{listings}
\usepackage{medmath}
\usepackage{hyperref}
\usepackage[noabbrev,capitalize]{cleveref}

\let\origsubsection\subsection
\renewcommand{\subsection}{\FloatBarrier\origsubsection}

\DeclareMathOperator{\msum}{\medmath\sum}

\title{Coordinate-Free Automatic Differentiation}

\author{Alex Guerra$^1$ \quad Joshua Aduol$^1$ \quad Yaqian Tang$^1$ \quad Daniel Williams$^1$ \quad Ryan P. Adams$^1$ \quad Peter Orbanz$^2$}
\affiliation{%
  $^1$Department of Computer Science, Princeton University, Princeton, NJ 08540 \\
  \texttt{ag5008@princeton.edu}, \texttt{jaduol@princeton.edu}, \texttt{yt1893@princeton.edu}, \texttt{dw1978@princeton.edu}, \texttt{rpa@princeton.edu} \\[0.5em]
  $^2$Gatsby Computational Neuroscience Unit, University College London \\
  \texttt{porbanz@gatsby.ucl.ac.uk}
}

\begin{document}

\begin{abstract}
Automatic differentiation (AD) is a cornerstone of modern machine learning. Since all AD frameworks implicitly fix a basis of Euclidean space to define gradients and Jacobians, they natively process mathematical objects that are arrays of a fixed set of coordinates. This makes it awkward to express objects such as rotations, volume forms, and curvatures that arise, for example, in applications of AI to science. Bespoke workarounds lack reusability. We show that (i) a single, basis-independent derivative operator can represent many geometric differential operations, and (ii) this operator can be expressed entirely in terms of existing AD primitives. This reusable transformation can be composed with JAX batching, compilation, and hardware acceleration. Geometry enters as algebraic data. We experimentally demonstrate its reuse across different manifold operators, its composition through a spacetime inverse-problem pipeline, and its use in optimization over different rotor subgroups.
\end{abstract}

\maketitle

\section{Introduction}

Automatic differentiation (AD) plays a pivotal role in modern machine learning. On the one hand, it provides convenience and efficiency: forward computations efficiently and automatically compute gradients and tangents via the chain rule.
On the other hand, it is a powerful programming abstraction that turns differentiation into a function transformation: if a function is implemented using differentiable primitives, it can be turned into another function that computes its local linearization, its adjoint, or higher-order variants, and these transformed functions can be composed with batching, compilation, and optimization.
One-off calculations involving derivatives become reusable software objects~\citep{baydin2018automatic, jax2018github}.

Current AD systems are built around coordinate arrays, which work well if derivatives are standard Euclidean Jacobians; this is often the case for the high-dimensional scalar objectives common in deep learning.
Simple coordinate arrays do not, however, directly cover other geometric objects and operators, such as fields on manifolds, rotations and rigid motions, curvature, volume forms, divergence, curl, or covariant derivatives.
These objects depend on additional geometric structure, such as a metric, an orientation, or a tangent space, none of which is carried by the differential that AD computes. Consequently, in domains where such objects are central, including physics, graphics, and engineering, differentiation is often not fully automatic.
We address the problem by identifying a single, common source of the difficulty.
Because the differential carries no geometric structure, that structure must be supplied by hand via a representation: Christoffel symbols in a chart, cotangent weights on a mesh, a stiffness matrix for a basis.
Once this representation is fixed, the resulting operator is tied to it, and the same underlying operator must be re-derived for each new space.
From a mathematical point of view, basis dependence is not inherent: coordinate-free formalisms, such as differential forms and exterior calculus, avoid this by describing geometric objects and operators independently of any particular coordinates.

We extend the AD philosophy of composable function transformations from flat-space AD to a fully coordinate-free formalism. The formalism used here is known as geometric algebra~\citep{hestenes1999new, Doran_Lasenby_2003}. It (i) provides a coordinate-free multivector derivative that incorporates all relevant metric and orientation information algebraically, and (ii)
aligns with the geometric-algebra representations already used in parts
of machine learning~\citep{brandstetter2023clifford, ruhe2023geometric, brehmer2023geometric}. Concretely, our contributions are:
\begin{itemize}
\item We show how the multivector derivative can be represented in terms of JAX primitives. It can therefore be implemented as a composable JAX transformation that supports forward and reverse mode AD, higher-order differentiation, and batched evaluation.
\item By specializing this derivative, we obtain automatic computation of curl, divergence, covariant derivatives, shape operators, curvature tensors and related operators on any (pseudo-) Riemannian manifold, in any number of dimensions.
\item We demonstrate (i) the framework’s feasibility and (ii) reuse across applications in physics and machine learning, such as divergence-free flows on manifolds, electromagnetic inverse problems, and rotor-group optimization.
\end{itemize}
Just as \texttt{jax.grad} transforms a scalar function into its Euclidean gradient, our framework transforms a multivector function into its coordinate-free derivative.
Because it is built entirely from JAX primitives, it composes with other JAX transformations and models.

\section{Related Work}

\noindent{\bf Automatic differentiation}
computes derivatives of programs by systematic
application of the chain rule~\citep{baydin2018automatic, griewank2008evaluating}
and is implemented in modern systems such as JAX, PyTorch, and TensorFlow
\citep{jax2018github, paszke2019pytorch, abadi2016tensorflow}. More mathematical accounts make the role of coordinates explicit, and relate forward and reverse mode to tangent and cotangent constructions~\citep{betancourt2018geometric, elliott2018simple}. Our
work keeps the composable transformations of modern AD, but replaces the
Euclidean coordinate derivative with a coordinate-free geometric derivative.

\noindent{\bf Existing extensions}
of AD beyond plain Euclidean arrays include
forward-mode AD for multivector-valued functions
\citep{tingelstad2017automatic, tingelstad2017motor, tingelstad2018motor},
manifold and Lie-group AD methods that use e.g.~charts or retractions to differentiate structured spaces
\citep{sommer2013automatic, sola2021micro, deray2020manif}, and recent work on
jet functors and Weil algebras in JAX~\citep{anonymous2025jet}.
In these methods, the geometry enters through a chosen local representation of
the space being differentiated. In our method, the derivative is first defined
algebraically on multivectors; manifold structure enters afterward through
projection onto the tangent space.

\noindent{\bf Geometric algebra},
or Clifford algebra, combines the inner and outer products
into the geometric product
\citep{hestenes1966spacetime,hestenes1999new}. Geometric calculus develops the associated
multivector derivative and recovers many operators of Riemannian differential
geometry without charts or Christoffel symbols
\citep{hestenes1968multivector, lasenby1993multivector, doran1994geometric}.
It is closely related to other coordinate-free formalisms such as differential
forms and exterior calculus, and complementary to discretized approaches such as
DEC and FEEC
\citep{arnold2010finite,
  hirani2003discrete, desbrun2005discrete, crane2013digital}.
  A broad ecosystem of GA libraries exists, e.g.~\citep{hadfield2013clifford, bromborsky2019galgebra, kahlow2020jaxga, kahlow2020tfga,
  roelfs2025kingdon, hoogendoorn2025numga}.

\noindent{\bf Geometric algebra in machine learning}
Geometric algebra has appeared in machine learning as a representation language for equivariant architectures, including Clifford neural layers, geometric Clifford algebra networks, the Geometric Algebra Transformer, and Clifford-steerable CNNs
\citep{brandstetter2023clifford, ruhe2023geometric, ruhe2023clifford,
  brehmer2023geometric, dehaan2024euclidean, spinner2024lorentz,
  zhdanov2024clifford}. These works are not concerned with geometric calculus or
automatic computation of differential operators; rather, they show that
multivector representations are already useful in ML and provide natural
  downstream settings for our AD operators.

\section{The Multivector Derivative}

\label{sec:mv-derivative}

This section describes our approach.
We define the multivector derivative and then translate this derivative into AD primitives.
This allows us to compute a range of differential operators in ``flat'' spaces.
In Section \ref{sec:manifold-calculus}, we generalize the approach to manifolds via projections onto tangent spaces.

\subsection{Definitions}

We begin with an~${n}$-dimensional vector space~${V}$, with an inner product~${\langle\cdot,\cdot\rangle}$.
The \emph{outer product}~${a_1\wedge\cdots\wedge a_k}$ of~${k}$ vectors represents an oriented,~${k}$-dimensional volume element if the vectors are linearly independent (and vanishes otherwise), and is called a~${k}$-\emph{blade}.
A sum of~${k}$-blades is a grade~${k}$ multivector (also called a~${k}$-vector).
Geometric algebra adds an additional product operation to~${V}$, called the
\emph{geometric product} and given by
\begin{equation}
  \label{geometric:product}
    ab = \langle a, b\rangle + a\wedge b \qquad\text{ for any }a,b\in V\;.
\end{equation}
Observe this implies ${\langle a, b\rangle = \frac12(ab+ba)}$ and ${a\wedge b=\frac12(ab-ba)}$.
Closing~${V}$ under this product produces the \emph{geometric algebra}~${\Lambda(V)}$.
This is a vector space that contains the original vectors,
but also oriented area elements, volume elements, and sums of these objects.
If~${b_1,\ldots,b_n}$ is a basis of~${V}$, a basis of~${\Lambda(V)}$ can be obtained
by forming every possible~${k}$-blade from the~${n}$ basis vectors
\begin{equation*}
  b_J\;:=\;b_{j_1}\wedge\ldots\wedge b_{j_k}
  \quad\text{ for }J:=\{j_1,\ldots,j_k\}\subset[n]\;.
\end{equation*}
 We denote the subspace of grade-$k$ elements $\Lambda^k(V)$, and the grade-$k$ part of $X$ by ${\langle X\rangle_k}$. The inner product~${\langle\cdot,\cdot\rangle}$ extends from~${V}$ to all of~${\Lambda(V)}$ as the bilinear, scalar-valued form ${\langle A, B\rangle := \langle AB\rangle_0}$, and~${\wedge}$ extends by grade selection ${A\wedge B := \langle AB\rangle_{p+q}}$ for ${A\in\Lambda^p(V), B\in\Lambda^q(V)}$. We also define the \emph{contraction} $A \rfloor B := \langle AB\rangle_{q-p}$, and is $0$ if $q - p < 0$. These definitions are all standard; see e.g.~\citet{macdonald2017survey, hestenes1999new}.

Any element~${X}$ of~${\Lambda(V)}$ can now be represented by expanding it in the basis above as
\begin{equation*}
    X = \msum_{J\in\mathcal I} X_J b_J\qquad\text{ for some }\mathcal{I}\subset 2^{[n]}\;,
\end{equation*}
for scalar coefficients~${X_J}$, called \emph{coordinates}.
We can always, of course, choose~${\mathcal{I}=2^{[n]}}$ and set~${X_J=0}$ if~${b_J}$ does not contribute,
but explicitly using a smaller set~${\mathcal{I}}$ can be useful to restrict attention to certain types of
objects. For example, if we choose~${\mathcal{I}}$ as the set of singleton sets~${\{1\},\ldots,\{n\}}$, then~${X}$ is an element of the original vector space~${V=\Lambda^1(V)\subset\Lambda(V)}$.
\begin{remark}
We are interested in low dimensional geometry. In high dimensions, computational costs can become prohibitive, since the space
${\Lambda(V)}$ has dimension~${2^n}$.
\end{remark}

\subsection{The Multivector Derivative}

We consider \emph{multivector functions} ${F:\Lambda(V)\to\Lambda(V)}$, or functions defined on the linear subspace~${W\subset\Lambda(V)}$ spanned by the basis blades~${\{b_J\}_{J\in\mathcal{I}}}$ for some index set~${\mathcal{I}\subset 2^{[n]}}$. A nondegenerate~${\langle \cdot, \cdot \rangle}$ on~${W}$ implicitly defines a reciprocal basis~${\{b^J\}}$ satisfying ${\langle b^J, b_K\rangle = \delta_{JK}}$.
The \emph{multivector derivative} of~${F}$ indexed by~${\mathcal{I}}$ is defined as
\begin{equation}
  \label{eq:flat_derivative}
    \partial_{\mathcal I}F
    :=
    \msum_{J\in\mathcal{I}} b^J\mfrac{\partial F}{\partial X_J}
    \qquad\text{ for }\mathcal{I}\subset 2^{[n]}\text{ and }X=\msum_{J\in\mathcal{I}}X_Jb_J\;.
\end{equation}
The multiplication in the sum is the geometric product. Four properties of~${\partial_\mathcal{I}}$ are worth emphasizing:
\begin{itemize}
\item ${\partial_{\mathcal{I}}}$ is coordinate-free.
Eq.~\ref{eq:flat_derivative} depends only on this inner product (which defines the~${\{b^J\}}$),~${W}$, and the geometric product, but not on the particular choice of basis~${\{b_J\}_{J\in\mathcal{I}}}$ spanning~${W}$ (Appendix~\ref{sec:basis-indep}). The only requirement is that the inner product be non-degenerate on~${W}$.
\item $\partial_\mathcal{I} F$ is intrinsic to $\Lambda(V)$, in contrast to the Jacobian matrix of a function on $\mathbb{R}^n$, which lives in $\mathbb{R}^{n^2}$.  In fact, $\partial_\mathcal{I} F$ can be a defined as a projection $\textrm{End}(\Lambda(V)) \to \Lambda(V)$.
\item All differential operators relevant to our purposes are special cases of~${\partial_\mathcal{I}}$.
\item As we will see below,~${\partial_{\mathcal{I}}}$ can be mapped to JAX primitives.
\end{itemize}

\subsection{From the Definition to AD primitives} \label{sec:implementation}

The derivative above can be implemented using AD primitives for Jacobian-vector products.
In JAX, the relevant primitives are
\texttt{jvp} and \texttt{vjp}. We stylize notation slightly, and write~${\texttt{jvp}(F,X,Z)}$ for \texttt{jvp(F, (X,), (Z,))}, etc.
The Jacobian-vector product computed by forward-mode AD is the directional derivative, or differential applied to a tangent vector:
\begin{equation*}
  \texttt{jvp}(F,X,A)
  =
  dF_X[A]
  :=
  \left.\frac{d}{dt}F(X+tA)\right|_{t=0}
\end{equation*}
of~${F}$ at~${X}$ in direction~${A}$. Reverse mode is more subtle; the vector-Jacobian product is meant to pull back an element of~${\Lambda(V)^*}$, a linear functional on~${\Lambda(V)}$. While this pullback is naturally defined (metric-free), AD stores this map as coordinate data where it assumes~${\Lambda(V) = \Lambda(V)^\ast}$. For more general settings, we require the \emph{musical isomorphisms} ${\flat: \Lambda(V)\to \Lambda(V)^\ast}$ and ${\sharp :\Lambda(V)^\ast\to \Lambda(V)}$, which are the identity on coordinates only when all~${b_J}$ are their own reciprocals. With these maps, we can define a ``metric aware'' vector-Jacobian product as:
\begin{equation}
    \label{eq:vjpmetric}
    \texttt{vjp}_{\langle\cdot,\cdot\rangle}(F,X,C) := \texttt{vjp}(F,X,C^\flat)^\sharp = dF_X^\ast[C]\quad\text{where}\quad
  \big\langle dF_X^\ast[C],A\big\rangle=\big\langle C,dF_X[A]\big\rangle \,.
\end{equation}

With this reverse-mode primitive defined, we introduce the following elementary lemma, which shows how the multivector derivative arises from these primitives.

\begin{lemma}
  \label{lem:mv-jvp}
  Let ${F:\Lambda(V)\to\Lambda(V)}$ be a differentiable function. Let~${\{b_J\}_J}$ be basis blades, with reciprocal blades~${\{b^J\}_J}$, and~${\mathcal{I}\subset 2^{[n]}}$.
  If~${X=\sum_{J\in\mathcal I}X_J b_J}$ ranges over the linear subspace of~${\Lambda(V)}$ defined by~${\mathcal{I}}$, then
 \begin{equation*}
   \partial_{\mathcal I}F(X)
   =
   \msum_{J\in\mathcal I}
   b^J\,\text{\tt{jvp}}(F,X,b_J)\;.
 \end{equation*}
 Similarly, if~${\mathcal{O}}$ indexes a set of basis blades in the codomain of~${F}$, then
 \begin{equation*}
      \partial_{\mathcal O}F(X)
      =
      \msum_{K\in\mathcal O}
      \text{\tt vjp}_{\langle \cdot, \cdot \rangle}(F,X,b^K)\, b_K
      \;.
 \end{equation*}
\end{lemma}

See Appendix~\ref{sec:proofs} for the proof and Appendix~\ref{sec:api} for implementation details.
\begin{remark}
If~${\mathcal{I}}$ contains all input directions and~${\mathcal{O}}$ all output directions present in~${F}$, then~${\partial_\mathcal{I}=\partial_\mathcal{O}}$, but costs may differ: Forward mode makes~${O(|\mathcal{I}|)}$ {\tt jvp} calls, reverse mode~${O(|\mathcal{O}|)}$ {\tt vjp} calls. If~${F}$ is scalar, for example, then~${|\mathcal{O}| = 1}$ and a single {\tt vjp} suffices, recovering the familiar reverse-mode advantage of standard {\tt jax.grad}.
\end{remark}

\subsection{The Vector Derivative}

We are particularly interested in the case ${\mathcal{I}_1 :=\{\{1\},\ldots,\{n\}\}}$, where all indices are singletons. In this case,~${\partial_{\mathcal{I}_1}}$ becomes the \emph{vector derivative} ${\partial_{\mathcal{I}_1} := \sum_{i=1}^n b^i\,\partial/\partial x_i}$. If we take the geometric product of~${\partial_{\mathcal{I}_1}}$ and a grade-${k}$ function~${F}$, Eq.~\ref{geometric:product} shows that the result splits into exactly two grades,
\begin{equation} \label{eq:grade_decomp}
    \partial_{\mathcal{I}_1} F
    =
    \underbrace{\partial_{\mathcal{I}_1}\wedge F}_{\text{grade $k+1$}}
    +
    \underbrace{\partial_{\mathcal{I}_1}\rfloor F}_{\text{grade $k-1$}} .
\end{equation}
The grade-raising part~${\partial_{\mathcal{I}_1}\wedge F}$ is the exterior derivative.
If~${F}$ is scalar valued, this is the ordinary gradient. If~${F}$ is vector valued,~${\partial_{\mathcal{I}_1}\wedge F}$ is the curl-like component of~${F}$, and grade-lowering summand~${\partial_{\mathcal{I}_1}\cdot F}$ is the divergence-like component~\citep{hestenes1968multivector, lasenby1993multivector, hitzer2002multivector, macdonald2017survey}. The reason the case of the vector derivative is of primary importance is for computations on manifolds, as we expound on next.

\subsection{From Flat Spaces to Manifolds}
\label{sec:manifold-calculus}

The approach developed so far assumes a ``flat'' ambient space. A simple construction generalizes it to curved manifolds. Let~${M}$ be an~${m}$-dimensional manifold. We will represent~${M}$ as a ``vector manifold'' \citep{Hestenes_Sobczyk_1984, hestenes2011shape}, with points~${x \in M}$ represented as vectors (in~${\Lambda^1(V)}$),
and a tangent space of~${M}$ at~${x}$ given by a blade~${I(x)}$ of grade~${m}$.\footnote{Vector Manifolds motivate why the vector derivative is of primary concern for this work, as all functions on manifolds can now be viewed as functions from ${M\subset \Lambda^1(V)\to \Lambda(V)}$. One could in principle also attempt to study `bivector manifolds', etc., but this does not embed cleanly into differential geometry.} The projection of a function~${F:M\to\Lambda(V)}$ onto the tangent space at~${x}$ is given by
\begin{equation} \label{eq:projection}
    P(F)(x) := (F(x) \rfloor I^{-1}(x)) \rfloor I(x)
\end{equation}
A grade~${k}$ function~${F}$ is called a \textit{field} on~${M}$ if~${P(F) = F}$. By projecting the vector derivative operator~${\partial_{\mathcal{I}_1}}$, we obtain the \emph{intrinsic} derivative:
\begin{equation} \label{eq:relative_derivative}
  \partial_M := P(\partial_{\mathcal{I}_1})
\end{equation}
Unfortunately,~${\partial_M F}$ is not necessarily a field, even if~${F}$ is. This can be remedied by simply defining the \emph{covariant derivative} to be the projection of~${\partial_M F}$:
\begin{equation} \label{eq:covariant_derivative}
    \nabla_M F := P(\partial_M F) = P(P(\partial_{\mathcal{I}_1})F)\;
\end{equation}
The relevance of these definitions is:
\begin{itemize}
\item As we will see in more detail in Section \ref{sec:experiments}, standard differential operators on manifolds can be expressed in terms of~${\partial_M}$ and~${\nabla_M}$. The curl and divergence on~${M}$, for example, are of the form~${\nabla_M \wedge F}$ and~${\nabla_M \rfloor F}$. Both are nilpotent operators:~${(\nabla_M\wedge)^2 = 0}$ (the GA form of~${d^2 = 0}$ from differential forms) and dually~${(\nabla_M\rfloor)^2 = 0}$ (codifferential~${\delta^2 = 0}$). This ensures that the Hodge-De Rham Laplacian~${\Delta_M = \nabla_M^2}$ is \emph{grade preserving}, as ${(d + \delta)^2 = d^2 + d\delta +\delta d + \delta^2 = d \delta + \delta d}$.
\item The projector Eq.~\ref{eq:projection} is easy to compute in an AD framework.
\end{itemize}
In short, differential operations on manifolds reduce to AD primitives via
Eq.~\ref{eq:projection} and the lemma above.

\begin{remark}
Compositions of~${P}$,~${\partial_M}$, and~${\nabla_M}$ give rise to the standard objects of differential geometry on~${M}$ (intrinsic differential operators, Laplacians, and curvature tensors), with no Christoffel symbols, coordinate charts, or dimension-specific formulas required. A parallel integration theory pairs~${\partial_M}$ with the directed volume field~${\mathrm{d}X}$ (the tangent space~${I(x)}$ weighted by the Riemannian measure~${\lvert\mathrm{d}X\rvert(x)\in\mathbb{R}}$) to give a single multivector-valued fundamental theorem subsuming the classical Stokes-type results~\citep{Hestenes_Sobczyk_1984}. Appendix~\ref{sec:api_manifolds} develops these in full, showing the full `API' for differential geometry.
\end{remark}

\section{Experiments}
\label{sec:experiments}
We present four sets of experiments. The first two concern manifold learning, using more standard differential geometric constructions. The second two go beyond differential geometry and show the application of our framework in physics and constrained optimization.
These experiments are end-to-end demonstrations, not performance benchmarks. They show how a small collection of AD transformations can be reused and composed.

\subsection{Explicit Manifold Learning} \label{sec:explicit_manifold}
We can learn an explicit representation of a manifold by starting with a known simple one~${M_0}$ (a sphere, a disk, a flat square) and applying some smooth transformation to get a new manifold $M_1$. This set of experiments work within this paradigm to solve various learning problems.

\begin{figure}[!h]
\centering
\begin{minipage}[t]{0.48\textwidth}
\centering
\includegraphics[width=\linewidth]{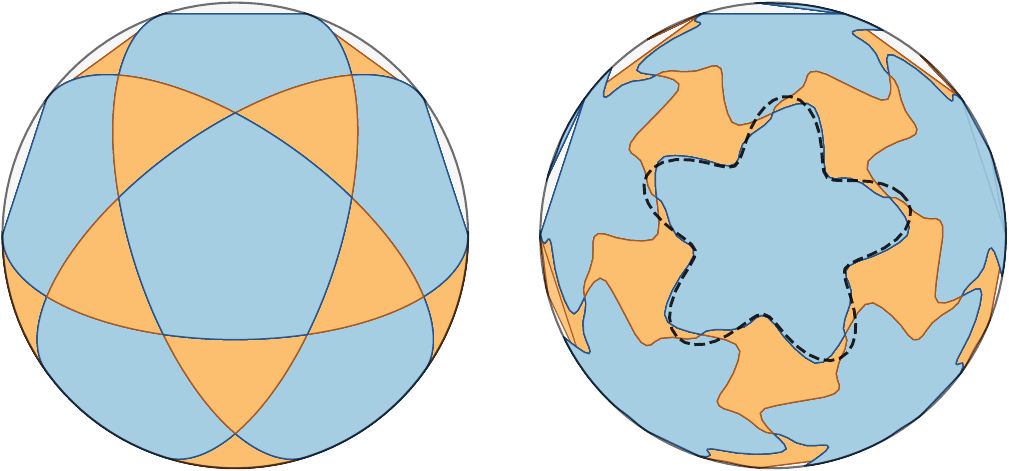}
\caption{Icosidodecahedron~${\to}$~${\mathbb{Z}_5}$ pinwheel on~${S^2}$. Target in black.}
\label{fig:exp2_tile_sphere_v2}
\end{minipage}\hfill
\begin{minipage}[t]{0.48\textwidth}
\centering
\includegraphics[width=\linewidth]{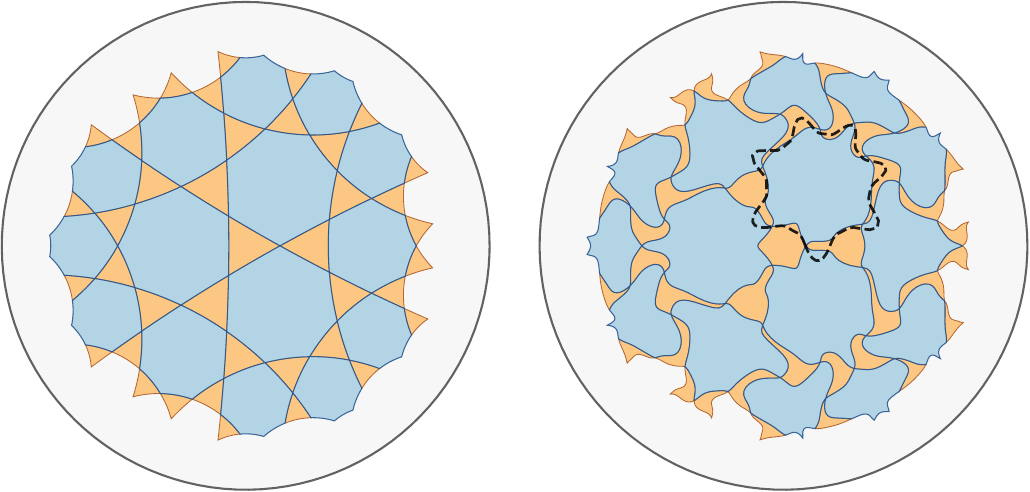}
\caption{Rectified~${\{3,7\}}$ tiling~${\to}$~${D_7}$ gear on $H^2$.}
\label{fig:exp2_tile_poincare_v2}
\end{minipage}
\end{figure}
\paragraph{Non-Euclidean tilings.} A general recipe for constructing a divergence-free vector field on a manifold is immediate in our framework. Start with an arbitrary 2-form, project to get a 2-field, and then apply~${\nabla_M\rfloor}$ to get a divergence-free vector field, recalling that~${(\nabla_M \rfloor)^2=0}$. This generalizes the Euclidean neural conservation law construction~\citep{richterpowell2022neuralconservationlawsdivergencefree}. In the particular case of a 2-surface, any 2-field must be a scalar multiple of~${I_M}$, so we parameterize our field via~${\phi_\theta I_M}$ for a scalar valued neural network~${\phi_\theta}$, with the resulting divergence-free field given by~${\nabla_M \rfloor (\phi_\theta I_M)}$. We can apply these divergence free flows to learn tilings on manifolds; this has proven useful in, e.g., the design of mechanical meta-materials ~\citep{mirramezani2024designingmechanicalmetamaterialslearning}. Our framework gives us a recipe to do this on any manifold: start with a fixed tiling, apply the flow, and then minimize a loss on the resulting tiles (ours is MSE against a target shape). On~${S^2}$ we deform an icosidodecahedron into a chiral,~${\mathbb{Z}_5}$-symmetric pinwheel target, see Fig.~\ref{fig:exp2_tile_sphere_v2}.

As a second example, on the Poincar\'e disk~${H^2}$ we deform a rectified~${\{3,7\}}$ tiling into a~${D_7}$-symmetric gear, see Fig.~\ref{fig:exp2_tile_poincare_v2}. With the ambient Minkowski space~${\mathbb{R}^{2,1}}$ spanned by~${\{e_1,e_2, e_n\}}$, our framework immediately grants us access to the differential operators on~${H^2}$ once we write down the manifold transformation~${\mathbb{R}^2\to H^2}$ given by $u\mapsto \cosh|u|\,e_n + (\sinh|u|/|u|)\,u$. The hyperbolic metric is induced by the ambient Minkowski inner product restricted to the tangent space of~${H^2}$. The divergence free flow here preserves area with respect to this metric, which explains the distortion near the boundary. This also makes it impossible for the learned shape to fully align with the target shape. For complete details of the setup and training, see Appendix~\ref{sec:tilings_details}.

\begin{wrapfigure}{r}{0.40\textwidth}
\centering
\includegraphics[width=0.38\textwidth]{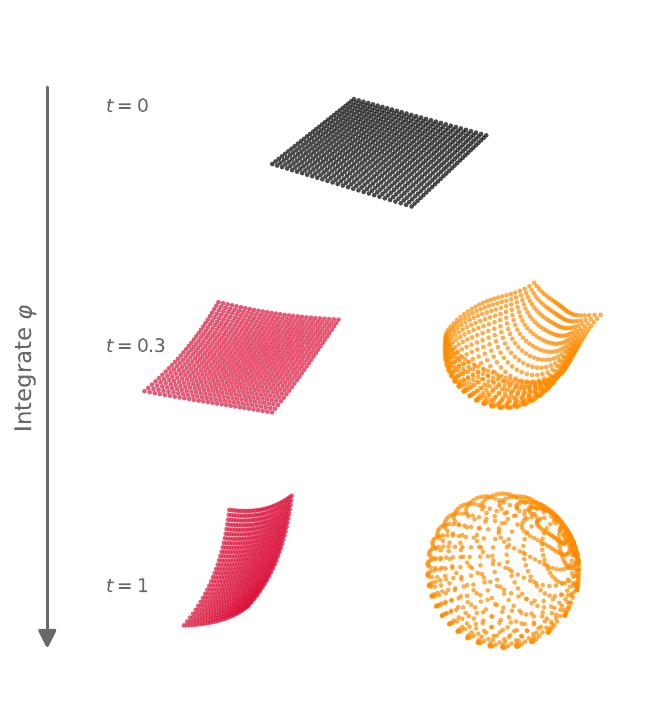}
\caption{Flat square (black) flowed under curvature loss (red) vs.\ with the added Gauss--Bonnet term (orange).}
\label{fig:exp2_euler_v2}
\end{wrapfigure}
\paragraph{Backprop through the Euler characteristic.} As a fun learning demonstration enabled by our framework, consider the problem of learning to warp a flat square into a sphere. We can parameterize this warping as the flow~${\varphi_t^\theta}$ of a neural ODE~\citep{chen2018neural} acting on our initial manifold~${M_0}$, and train it end to end so that the resulting manifold has certain geometric properties; see Appendix~\ref{sec:flow_matching_details} for details. One way to characterize the sphere is by its constant curvature. A remarkable fact is that nearly all curvature-based properties of a manifold can be computed as compositions of~${\partial_M}$ and~${I_M}$; Hestenes's so-called \emph{shape} operator~${S_M = \partial_M P}$ measures how the projector changes as we move along the manifold, which turns out to encode a lot of information about the local geometry. In particular, scalar curvature~${R}$ can be read off almost immediately for a manifold of any dimension, and equals twice the Gaussian curvature,~${R = 2K}$, on a surface. Training a flow to match the sphere's curvature results in learning a small section of the sphere (see Fig.~\ref{fig:exp2_euler_v2}). To warp the square into a full sphere, we can augment this training objective with a \emph{differentiable} estimate of the Euler characteristic~${\chi_M}$ of~${M}$ via the Gauss--Bonnet theorem for closed surfaces,
\begin{equation} \label{eq:gauss_bonnet}
    \chi_M \;=\; \frac{1}{4\pi}\int_M R \,\lvert\mathrm{d}X\rvert \;\approx\; \frac{1}{4\pi}\sum_{i=1}^N R(x_i)\,\lvert\mathrm{d}X\rvert(x_i).
\end{equation}
Again, this is not only easy to compute in our library, but differentiable, so we can \emph{directly optimize} Eq. \ref{eq:gauss_bonnet} to approach the Euler characteristic of the sphere~${\chi_{S^2} = 2}$. Of course, the square is not a closed manifold, and~${\chi}$ is invariant under diffeomorphisms (and hence the parameterized flow~${\varphi^\theta}$). Nevertheless, this added loss is still able to guide the flow to something that is approximately closed, from which the Euler characteristic estimate is approximately 2 (see Fig.~\ref{fig:exp2_euler_v2}).

\subsection{Implicit Manifold Learning} \label{sec:implicit_manifold}

Explicit manifold learning (\S\ref{sec:explicit_manifold}) requires a smooth invertible map~${\varphi}$ between a known base manifold and the target, which restricts the types of architectures one can use. We can also learn an \emph{implicit} representation, encoding the manifold as the support (or zero level set) of a learned field~\citep{park2019deepsdf, palmer2022deepcurrentslearningimplicitrepresentations}.

\begin{wrapfigure}{r}{0.38\textwidth}
\centering
\includegraphics[width=0.18\textwidth]{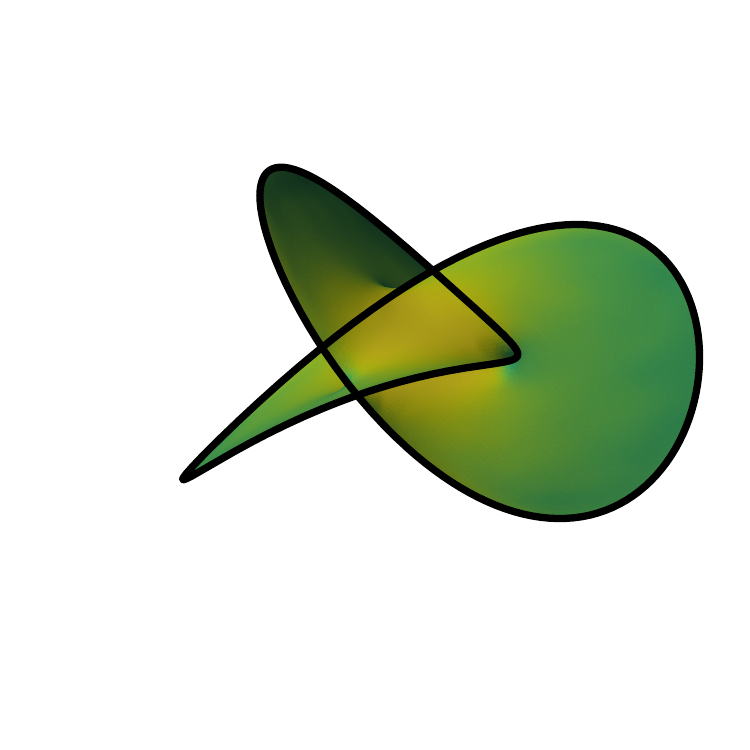}\hfill
\includegraphics[width=0.18\textwidth]{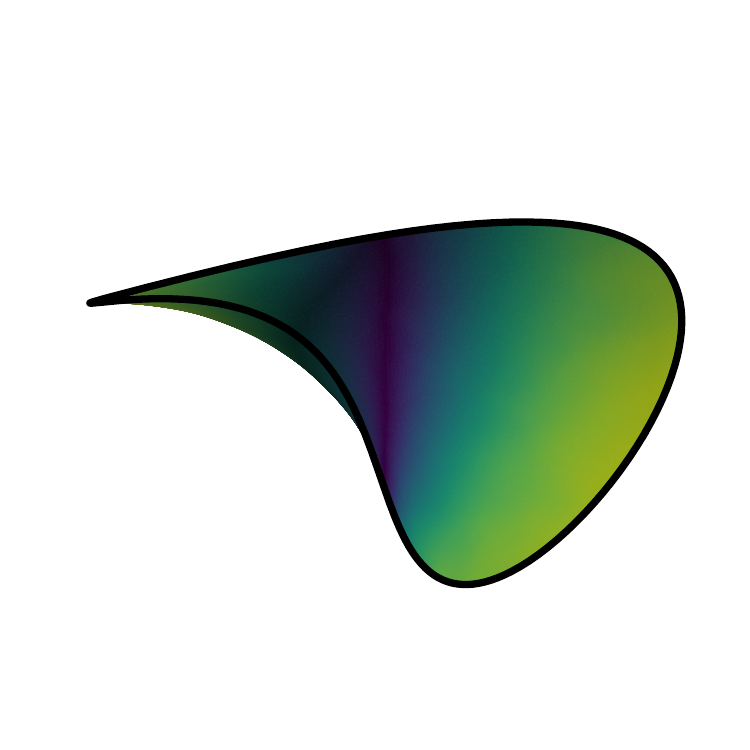}
\caption{Learned minimal surfaces for trefoil (left) and a saddle (right).}
\label{fig:exp3_minsurf_v2}
\end{wrapfigure}
\paragraph{Minimal surfaces and the Biot--Savart kernel.} Given a closed boundary ${\Gamma}$ of dimension ${n-k-1}$ in ${\mathbb{R}^n}$, we learn the bounding~${(n-k)}$-submanifold~${M_\theta}$ as the set on which a learned~${k}$-form ${\omega_\theta = \omega_\Gamma + \partial_{\mathcal{I}_1}\wedge \eta_\theta}$ concentrates (extracted from~${\omega_\theta}$ as a level set or ridge, Appendix~\ref{sec:minsurf_details}) for an arbitrary~${(k{-}1)}$-form~${\eta_\theta}$ parameterized by a neural network, refining the implicit-current representation of DeepCurrents~\citep{palmer2022deepcurrentslearningimplicitrepresentations}. The key observation is that the boundary constraint~${\partial M_\theta = \Gamma}$ can be made structural through the GA-native \emph{Biot--Savart kernel}:
\begin{equation}
\omega_\Gamma(x) = I^{-1}(x) \int_\Gamma \partial_{\mathcal{I}_1} G(x-y) \wedge d\Gamma(y) , \quad \partial_{\mathcal{I}_1} \wedge \omega_\Gamma = \delta_\Gamma, \label{eq:biotsavart}
\end{equation}

\noindent where~${G}$ is the fundamental solution to~${\Delta G = \delta}$, also known as the Green's function of the Laplacian.
With this choice, we guarantee that~${\partial_{\mathcal{I}_1} \wedge \omega_\theta = \delta_\Gamma}$, hence this parameterization satisfies~${\partial M_\theta = \Gamma}$ for any~${(n, k)}$ pair. DeepCurrents shows this for ${(n, k) = (3, 1)}$, which we show is easily replicated in our framework; see Fig.~\ref{fig:exp3_minsurf_v2}. The problem of generalizing the training to other~${(n, k)}$ is actually obstructed by the difficulty of computing the \textit{mass norm} of the learned current, as we discuss next. Intuitively, the mass norm is a generalization of the measure of the manifold, which is why minimizing this quantity gives surfaces of minimal measure. Next, we discuss how the mass norm is computed pointwise for~${k}$-forms; from this we can train the network by minimizing the integral of the mass norm over some ambient domain containing our surface.

\begin{figure}[!b]
\centering
\begin{minipage}[t]{0.48\textwidth}
\centering
\includegraphics[width=\linewidth]{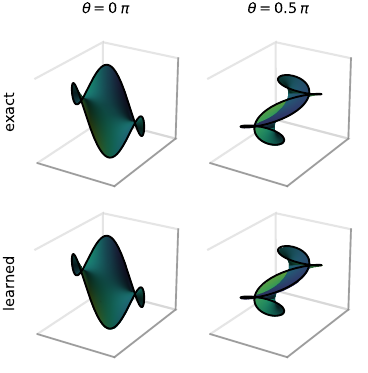}
\caption{A holomorphic graph over a disk in~${\mathbb{C}^2 \cong \mathbb{R}^4}$: exact minimal surface (top) vs.\ surface learned by minimizing the mass norm (bottom), at two projection angles.}
\label{fig:exp3_codim2_v2}
\end{minipage}\hfill
\begin{minipage}[t]{0.48\textwidth}
\centering
\includegraphics[width=\linewidth]{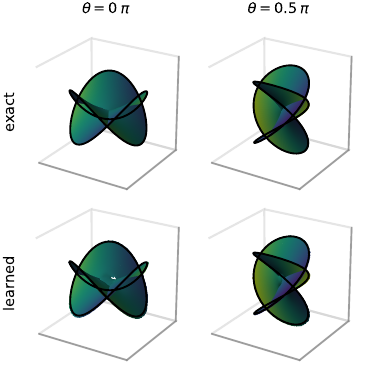}
\caption{A branched holomorphic disk bounding the~${(2,3)}$ torus knot, exact (top) vs.\ learned (bottom), as in Fig.~\ref{fig:exp3_codim2_v2}.}
\label{fig:exp3_codim2_torus_v2}
\end{minipage}
\end{figure}

The general way to define the (pointwise) mass norm of a~${k}$-form~${\omega}$ is to write~${\omega}$ as a sum of blades and then take the sum of the multivector norms of these blades. The true mass norm is the infimum of this quantity over all such decompositions. For a~${1}$-form~${\omega}$, all vectors are blades hence the mass norm coincides with the multivector norm. For a~${2}$-form~${\omega}$, the infimum is attained by writing~${\omega}$ as a sum of \textit{commuting} simple bivectors \citep{Hestenes_Sobczyk_1984}: ${\omega = \sum_{i=1}^{\lfloor n / 2 \rfloor} \lambda_i B_i}$ with~${B_iB_j=B_jB_i}$. The mass norm is then~${\lVert \omega \rVert =\sum_i \lvert \lambda_i \rvert}$. The essence of this norm is that it penalizes bivectors which are \textit{not} pure subspaces (blades).

Thus, a bivector like ${B = e_1\wedge e_2 + e_3\wedge e_4}$ is penalized for being a superposition of two pure subspaces (${e_1\wedge e_2}$ and~${e_3\wedge e_4}$) and has mass norm 2. In contrast, the bivector ${B' = e_1\wedge e_2 + e_3\wedge e_1=e_1\wedge(e_2 - e_3)}$ \textit{is} a pure subspace having mass norm~${\sqrt{2}}$. Under the multivector norm,~${B}$ and~${B'}$ would have the same norm. Minimizing the mass norm thus pushes~${\omega}$ to be simple almost everywhere. Importantly, the mass norm on bivectors can be computed efficiently and is differentiable almost everywhere.  Unfortunately, no closed form is known for the mass norm when~${3\leq k \leq n-3}$, so we only attempt the~${n=4,k=2}$ case. To get an interesting family of ground truth minimal surfaces for this codimension 2 setting, we can note that any holomorphic curve~${M}$ in~${\mathbb{C}^2}$ is mass-minimizing \citep{federer1969gmt}, hence defines a minimal surface for any boundary contained on it. We show these comparisons in Figs.~\ref{fig:exp3_codim2_v2} and~\ref{fig:exp3_codim2_torus_v2}. For the exact surfaces and their parameters, more examples, and detailed metrics, refer to Appendix~\ref{sec:minsurf_details}.

\subsection{Inverse Design in Spacetime} \label{sec:em}

\begin{wrapfigure}{l}{0.42\textwidth}
\centering
\includegraphics[width=0.40\textwidth]{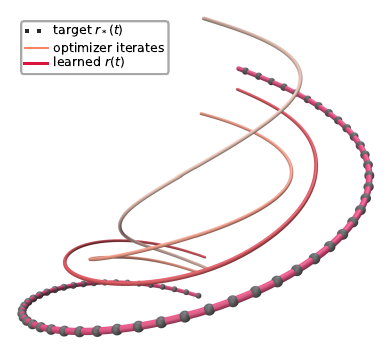}
\caption{Recovered trajectory~${r(t)}$ (red) against the target helix, with intermediate training trajectories also plotted.}
\label{fig:em_inv_v2}
\end{wrapfigure}
The spacetime algebra~${\text{Cl}(1,3)}$ is a useful model for electromagnetism ~\citep{hestenes1966spacetime, Doran_Lasenby_2003}. In this algebra, the electromagnetic field collapses into a single bivector~${F = E + IB}$ (electric~${E}$, magnetic~${B}$), the source 4-current is a single vector~${J}$, and the four classical Maxwell equations collapse to a single famous equation involving the multivector derivative:~${\partial_{\mathcal{I}_1} F = J}$. The grade decomposition splits this into the source half~${\partial_{\mathcal{I}_1}\rfloor F = J}$ and the homogeneous half~${\partial_{\mathcal{I}_1}\wedge F = 0}$ (no magnetic monopoles). Applying the multivector derivative once more gives charge conservation \emph{for free}:~${\partial_{\mathcal{I}_1}\rfloor J = 0}$. The 4-potential factorization~${F = \partial_{\mathcal{I}_1}\wedge A}$ in turn exists because~${(\partial_{\mathcal{I}_1}\wedge)^2 = 0}$, and \emph{gauge freedom} is the statement that~${A}$ and~${A + \partial_{\mathcal{I}_1} \wedge\phi}$ produce the same~${F}$ for any scalar~${\phi}$. Picking a gauge (e.g.~Lorenz,~${\partial_{\mathcal{I}_1}\rfloor A = 0}$) just selects a representative (usually~${\phi = 0}$).

\begin{wrapfigure}[18]{r}{0.32\textwidth}
\centering
\includegraphics[width=0.30\textwidth]{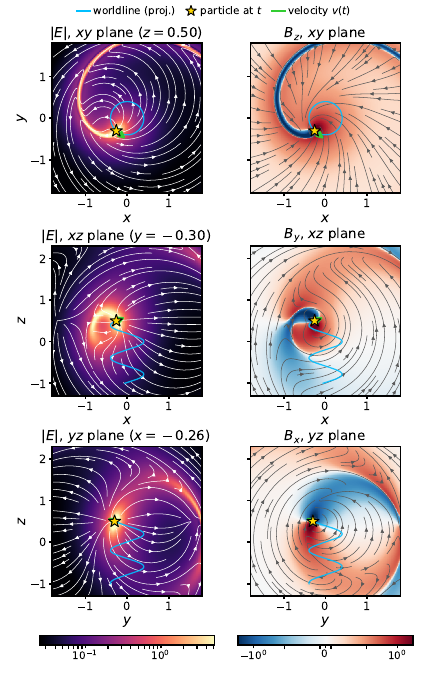}
\caption{Projections of ${F}$ field at~${t{=}2}$.}
\label{fig:em_fwd_v2}
\end{wrapfigure}
To demonstrate our framework in an inverse design problem, consider the problem of recovering a charged particle's trajectory~${r(t)}$ from a sparse set of Faraday bivector observations~${F(y_i)}$ at spacetime points~${y_i}$.

To parameterize the particle's worldline, we use a small MLP which outputs a 3-velocity~${v_\theta(t)}$ in the lab frame such that~${|v_\theta(t)| < 1}$ (so the predicted worldline is timelike and future-directed at every step). Integrating~${v_\theta(t)}$ gives the spatial trajectory~${r_\theta(t)}$, and hence the worldline~${x_\theta(t) = t\,e_0 + r_\theta(t)}$. To get the Faraday bivector~${F}$ at an observation point~${y}$, we first must find the retarded time~${\tau_r(y)}$ such that~${y-x_\theta(\tau_r(y))}$ is null, i.e., ${(y-x_\theta(\tau_r(y)))^2 = 0}$. Here~${\tau_r(y)}$ is the lab time at which the particle emits the signal received at~${y}$. We compute~${\tau_r(y)}$ via bisection search, which can be differentiated (with respect to~${r}$) via implicit function theorem. Once we have~${\tau_r(y)}$, we can compute the Li\'enard--Wiechert potential~${A(y)}$ and then the Faraday bivector ${F(y) = \partial_{\mathcal{I}_1}\wedge A(y)}$. The entire pipeline from~${v_\theta(t)}$ to~${F(y)}$ is differentiable, allowing us to backpropagate the relative MSE between predicted and observed Faraday bivectors to update the parameters of the MLP. The~${F}$ field at a point is shown in Fig.~\ref{fig:em_fwd_v2} for a particle on a helical worldline, and the learned trajectory is shown in Fig.~\ref{fig:em_inv_v2}. Full details of the experiment, including the observation grid, input encoding, and optimizer schedule, are in Appendix~\ref{sec:em_details}.

\subsection{Optimization on Rotor Groups}
\label{sec:rotoropt}
Until now, all our examples have used the \emph{vector} derivative $\partial_{\mathcal{I}_1}$; we now show the utility of a different instance of the multivector derivative for optimization on rotor groups and the homogenous spaces they act on. While this is a well studied topic~\citep{absil2008optimization, boumal2023intro, sola2021micro, deray2020manif}, a GA-native medium allows for these spaces to be specified as data instead of code, making optimization procedures reusable and composable.
\begin{figure}[!b]
\centering
\includegraphics[width=\textwidth]{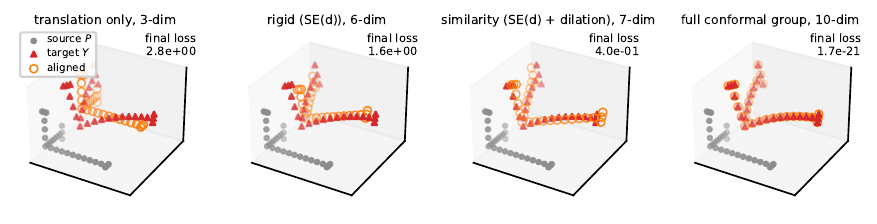}
\caption{Vectorized optimization over different subgroups: Given an object (gray) and a target conformal transformation (red), we simultanously learn 4 rotors in \textit{different} rotor groups corresponding to maskings of the 10-dimensional bivector space of~${\mathrm{CGA}(3)}$: translation (3D), rigid (6D), similarity (7D), and conformal (10D).}
\label{fig:lie_subalgebras}
\end{figure}

The \emph{rotors} of a geometric algebra are even products of unit vectors $R=\prod_{i=1}^{2k} u_i$ satisfying $R\tilde{R} = 1$. It is clear that rotors form a group under the geometric product called the rotor group $\mathrm{Spin}(p,q)^+\subset \mathrm{Cl}(p,q)$. Homogeneous spaces like the Stiefel manifold or the Grassmannian can be realized as $R X_0 \tilde{R}$ for some fixed multivector $X_0$. Thus, optimizing a function over one of these spaces generally reduces to the problem of minimizing a function $f: G\to \mathbb{R}$ for $G$ some subgroup of $\mathrm{Spin}(p,q)^+$.

To this end, we first need to specify $G$. One way to specify a group is via a set of generators. However, because $G$ is a \emph{Lie} group, we can instead provide generators of the corresponding Lie algebra $\mathfrak{g} := T_e G$. The utility of GA is that this Lie algebra will always be a \emph{bivector} subspace, leading to a compact and uniform way to capture all matrix Lie groups (subgroups of $\mathrm{Spin}(p,q)^+$). With this choice of bivectors $\{B_k\}$, suppose that we currently have a rotor $g_t \in G$ and we would like to find an update $U_t$ such that $f(g_t U_t) < f(g_t)$. This can be done with the \emph{left trivialized differential} $df_e (g_t)\in\mathfrak{g}^\ast$, which can be implemented in terms of the multivector derivative as:
\begin{equation}\label{eq:lie_deriv}
  df_e(g)= \partial_{\mathcal{I}_{\mathrm{even}}} f(g)\,g \in \mathfrak{g}^*,
\end{equation}
Intuitively, we can think of the multivector derivative as an element of $T^\ast_g G$, which we then right multiply by $g$ to get to the appropriate cotangent space at the identity, $\mathfrak{g}^*$. To turn this covector into an update direction we need a map $\sharp:\mathfrak{g}^*\to\mathfrak{g}$. For descent to work, the inner product inducing this map \emph{must} be positive definite, meaning the usual multivector inner product won't work. While AD implicitly uses the Euclidean inner product on coordinates, our framework lets any positive-definite choice be made explicitly. Once we have mapped our element to $\mathfrak{g}$, exponentiating gives:
\begin{equation}
    g_{t+1} = g_t \exp(-\eta\, df_e(g_t)^\sharp)
\end{equation}
This update is valid for any Lie group embedded in $\textrm{Cl}(p,q)$. Existing tools like \texttt{manopt}~\citep{boumal2014manopt} and \texttt{geoopt}~\citep{kochurov2020geoopt} implement separate code for each manifold being optimized over, and additionally only work out of the box for Lie groups which are compact or the semidirect product of a compact Lie group with $\mathbb{R}^n$ (e.g., $SE(n)=SO(n)\ltimes \mathbb{R}^n$). \texttt{geotorch}~\citep{lezcano2019trivializations} implements a similar trivialization procedure to ours, but even this library still assumes definite signature. See Appendix~\ref{sec:rotor_details} for further exposition.

With `groups as data', we can use JAX's vectorization to simultaneously optimize an objective across \emph{different} groups. Consider the chain $T(3)\subset SE(3)\subset \text{Sim}(3)\subset\text{Conf}(3)=\textrm{Spin}(4,1)^+$, which are all rotor subgroups of the conformal geometric algebra $\text{CGA}(3)=\text{Cl}(4,1)$: each Lie algebra in the chain is a coordinate subspace of the 10-dimensional bivector algebra of the full conformal group. In Figure~\ref{fig:lie_subalgebras}, we optimize the same point-cloud alignment loss on all four groups \emph{at once}: each panel converges to the best approximation its subgroup admits (matching the centroid, Procrustes, and Umeyama closed forms), and the conformal panel recovers the target configuration exactly.

Next, we show how this same framework can support optimization over the Lorentz group of the spacetime algebra. Suppose that we are on a ship moving at some uniform \textit{unknown} velocity and orientation, from which we can measure the (possibly noisy) electromagnetic field $\tilde{F}$. We can then try to recover the Lorentz transformation $R\in \mathrm{Spin}(1,3)^+$ that maps the lab frame to the ship frame by minimizing the loss between the mapped analytic field $R F_k \widetilde{R}$ and the measured field $\tilde{F}_k$. We plot the results of doing such an optimization in Figure~\ref{fig:lorentz_recovery}. The optimization converges to the correct Lorentz transformation, and we can recover the ship's velocity from the optimized rotor $R$. See Appendix~\ref{sec:rotor_details} for full training details.

\begin{figure}[t]
  \centering
  \includegraphics[width=0.8\textwidth]{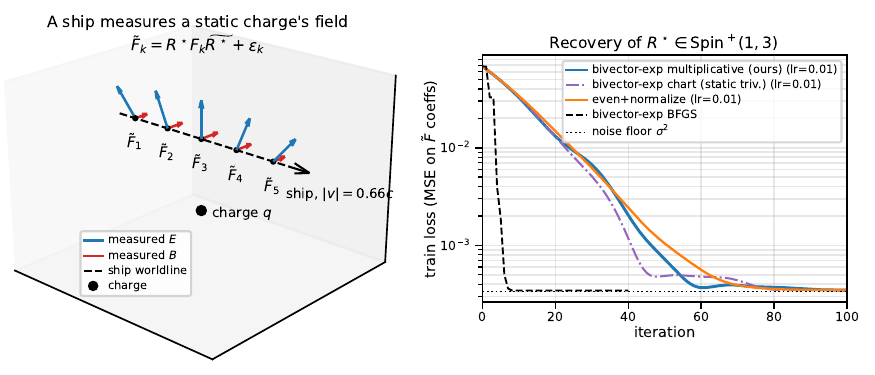}
  \caption{Recovering a Lorentz rotor from field measurements. \emph{Left:} a ship crosses the lab frame at uniform velocity (dashed), past a static point charge, measuring the Faraday bivector $\tilde{F_k}$ in its own frame at events along its path, ${\tilde{F}_k = R^\star F_k \widetilde{R^\star} + \epsilon_k}$. In the ship frame the pure-electric Coulomb field tilts and compresses (electric part, blue) and acquires a circulating magnetic part (red). \emph{Right:} training loss for recovering~${R^\star}$.}\label{fig:lorentz_recovery}
\end{figure}

\section{Discussion and Conclusion}

The experiments illustrate that a small set of geometric operators can support several apparently different workflows. The same implementation of the multivector derivative, exterior derivative, and tangent projection solves problems ranging from manifold learning, inverse problems in spacetime, and optimization on rotor groups. In each case, the geometric structure enters through the selected index set~${\mathcal I}$, the multivector field being differentiated, and, on manifolds, the tangent pseudoscalar~${I_M}$. This is the main practical benefit of the approach: geometric differential operators can be exposed as reusable AD transformations rather than re-derived for each coordinate system or application.

\subsection*{AI use statement}

In this work, we used generative AI tools for three primary tasks: first, to collaborate and iterate on experiments. While early versions of all of the experiments were entirely human written, we found AI to increase our ability to iterate fast and find the right training settings to make things work. Second, we used AI to find relevant sources in the literature. We have carefully reviewed all of these sources to ensure that no citations are hallucinated. Third, we used generative AI to help revise the paper, improve the precision of mathematical statements, and for formatting purposes. The authors take full responsibility for all claims made in the paper.

\bibliographystyle{plainnat}
\bibliography{references}

\newpage
\appendix

\section{Proofs}
\label{sec:proofs}

Throughout this section, we let $V$ be an $n$-dimensional real vector space with a nondegenerate symmetric bilinear form of signature $(p, q)$, $\Lambda(\mathbb{R}^{p,q}) = \mathrm{Cl}(p,q)$ its geometric algebra, and $\{b_J\}_{J\subset [n]}$ a basis of blades induced by a basis $b_1, \ldots, b_n$ of $V$. Given the multivector inner product $\langle \cdot, \cdot \rangle$, we can compute the reciprocal basis blades by first computing the gram matrix $\mathbf{G}_{JK}=\langle b_J, b_K \rangle$, and then setting $b^J = \sum_K \mathbf{G}^{-1}_{JK} b_K$. Note that that is how we implement the musical isomorphisms $\sharp, \flat$ that we use to define our metric aware \texttt{vjp}. As a final remark, we mention that even for an orthonormal basis in a Euclidean GA $\Lambda(\mathbb{R}^n)$, while the Gram matrix on vectors on vectors might be the identity, the full gram matrix $\mathbf{G}$ will not: consider, for example, the basis vector $e_{12}=e_1\wedge e_2$. Then, $e^{12} = -e_{12}$. Nondegeneracy of $\langle \cdot, \cdot \rangle$ ensures that $\mathbf{G}$ is invertible.

Another frequent fact that we will use throughout is that we can write a multivector as coordinates in some basis. For any multivector $A$, we have:
\begin{equation} \label{eq:dual_expansions}
  A \;=\; \msum_J \langle b^J, A\rangle\, b_J \;=\; \msum_J \langle A, b_J\rangle\, b^J
  \qquad\text{for all }A\in\Lambda(V).
\end{equation}
In general, this is how we can express a multivector $A$ in coordinates relative to a basis; we will use this expansion throughout.

\subsection{Basis Independence of \texorpdfstring{$\partial_{\mathcal{I}}$}{the Multivector Derivative}}
\label{sec:basis-indep}

\begin{proposition} \label{prop:basis_indep}
Let $\mathcal I\subset 2^{[n]}$ and take $W:=\operatorname{span}\{b_J\}_{J\in\mathcal I}=\operatorname{span}\{b^J\}_{J\in\mathcal I}$. The operator
$\partial_{\mathcal I}F=\sum_{J\in\mathcal I} b^J\,\partial F/\partial X_J$ depends only on the subspace $W$ and the pairing, not on the basis $\{b_J\}_{J\in\mathcal I}$ chosen for $W$.
\end{proposition}

\begin{proof}
Suppose that we had a different basis $\{c_J\}_{J\in \mathcal I}$ of $W$. By virtue of $\{b_J\}$ being a basis for $W$, every $c_J$ must be expressible as $c_J = \sum_{K} \mathbf{A}_{JK} b_K$ for some matrix $\mathbf{A}$. The reciprocals $\{b^J\}$ also form a basis, and thus should satisfy $c^J = \sum_{K} \mathbf{D}_{JK} b^K$ for some matrix $\mathbf{D}$. We can then note:
\begin{align}
    \delta_{JL} &= \langle c^J, c_L \rangle = \sum_{K,M} \mathbf{D}_{JK}\mathbf{A}_{LM} \langle b^k, b_M\rangle \\
    &= \sum_K \mathbf{D}_{JK} \mathbf{A}_{LK} = (\mathbf{DA}^\top)_{JL}
\end{align}
Therefore, $\mathbf{D} = (\mathbf{A}^\top)^{-1}$.

Next, note that for an arbitrary multivector $X\in W$, we have:
\begin{align}
    X &= \sum_J X'_J c_J \\
    &= \sum_J X'_J \left( \sum_K A_{JK} b_K \right) \\
    &= \sum_K \left(\sum_J A_{JK} X'_J \right) b_K \\
    &= \sum_K X_K b_K
\end{align}
This relates the coordinates $X_K$ and $X'_J$, and immediately implies:
\begin{align}
\frac{\partial X_K}{\partial X'_J} = A_{JK}
\end{align}
With this established, we can easily simplify the following multivector derivative expressed in the $\{c_J\}$ basis as one expressed in the $\{b_J\}$ basis:
\begin{align}
    \partial F &= \sum_J c^J \frac{\partial F}{\partial X'_J} \\
    &= \sum_J \left[\sum_{K} b^K (A^\top)^{-1}_{JK} \right] \left[ \sum_L \frac{\partial F}{\partial X_L} \frac{\partial X_L}{\partial X'_J}\right] \\
    &= \sum_{K,L} b^K \left[\sum_J (A^\top)^{-1}_{JK} A_{JL} \right] \frac{\partial F}{\partial X_L} \\
    &= \sum_{K,L} b^K \delta_{KL} \frac{\partial F}{\partial X_L} \\
    &= \sum_K b^K \frac{\partial F}{\partial X_K}
\end{align}
\end{proof}

\begin{remark}
The condition that the $\{b^J\}$ span the same subspace as $\{b_J\}$ can be relaxed, but convolutes the proof slightly. As an example of such an instance, consider the conformal euclidean algebra $\Lambda(\mathbb{R}^{n+1,1})$ with basis $e_1, \ldots, e_n, e_+, e_-$, and then consider the change of basis $e_\infty = e_+ + e_-$ and $e_o = \frac12 (e_- - e_+)$. It is easy to verify that $\langle e_\infty, e_o \rangle = -1$, meaning that the reciprocal of $e_\infty$ is $-e_o$. We can thus see that any space $W$ containing $e_\infty$ components without the $e_o$ components will have a reciprocal space that is not contained within it.
\end{remark}

\subsection{Proof of Lemma~\ref{lem:mv-jvp}}
\label{lemma-proof}

\begin{proof}[Forward-mode proof]
Take $X = \sum_{J\in\mathcal I} X_J b_J \in W$. Then by definition,
\begin{equation}
\frac{\partial F}{\partial X_J} = \frac{d}{dt} F(X + t b_J) \big|_{t=0} = dF_X[b_J] = \texttt{jvp}(F, X, b_J)
\end{equation}
so $\partial_\mathcal{I} F(X) = \sum_{J\in\mathcal{I}} b^J \texttt{jvp}(F, X, b_J)$.
\end{proof}

\begin{proof}[Reverse-mode proof]
Letting $\mathcal{O}$ index the output blades, so that $F(X) \in\operatorname{span}\{b_K: K\in \mathcal{O}\}$ for all $X$. Take $\mathcal{I}$ to be the corresponding input index set over which $X$ varies. Note that the adjoint is defined implicitly by the property $\langle C, dF_X[A] \rangle= \langle dF^\ast_X[C], A\rangle$ for the multivector inner product $\langle \cdot, \cdot \rangle$:
\begin{align}
    \partial_\mathcal{I} F &= \sum_{J\in\mathcal{I}} b^J dF_X[b_J] \\
    &= \sum_{J\in\mathcal{I}} b^J \left[ \sum_{K\in\mathcal{O}} \langle b^K, dF_X[b_J]\rangle b_K\right] \\
    &= \sum_{J\in\mathcal{I}} b^J \left[ \sum_{K\in\mathcal{O}} \langle dF_X^\ast[b^K], b_J\rangle b_K\right] \\
    &= \sum_{K\in\mathcal{O}} \left[ \sum_{J\in\mathcal{I}} \langle dF_X^\ast[b^K], b_J\rangle b^J \right] b_K \\
    &= \sum_{K\in\mathcal{O}} dF_X^\ast[b^K] b_K
\end{align}
and we are done as our metric aware $\texttt{vjp}_{\langle\cdot, \cdot \rangle}$ computes $dF^*$ with respect to the multivector inner product instead of the Euclidean one.
\end{proof}

\begin{remark}
If we assume that $F$ varies over the basis index set $\mathcal{O}$, and takes inputs $X$ which vary over the basis index set $\mathcal{I}$, then computing the forward mode $\partial_{\mathcal{I}}$ requires $\lvert \mathcal{I} \rvert$ \texttt{jvp} calls and geometric products. Computing the reverse mode $\partial_\mathcal{O}$ requires $\lvert \mathcal{O} \rvert$ \texttt{vjp} calls and geometric products. This implies that for computing multivector derivatives of scalar functions is most efficient in reverse mode, while most vector derivatives are more efficient to compute with the forward mode form since they only require $n$ \texttt{jvp}'s.
\end{remark}

\section{Software API and Implementation}
\label{sec:api}

Our implementation is written in JAX \citep{jax2018github} to be natively differentiable, and begins with the basic means of constructing Geometric Algebras and the operators on them. This base functionality is similar to that of other libraries such as \texttt{jaxga}, \texttt{clifford}, \texttt{numga}, and \texttt{galgebra}. However, once we get into the more calculus heavy operations, the autodiff capability of JAX makes it possible to go further than what these libraries provide.

\subsection{Algebras and Multivectors}
\label{sec:api_algebra}

In this work, we only consider geometric algebras over real vector spaces (though in principle we could support complex vector spaces as this is natively supported in JAX). Such an algebra is specified by nonnegative integers $(p,q,r)$, with $p$ denoting the number of positive basis vectors ($e^2=1$), $q$ denoting the number of negative basis vectors ($e^2=-1$), and $r$ denoting the number of zero basis vectors ($e^2=0$). We write this geometric algebra as $\Lambda(\mathbb{R}^{p,q,r})=\mathrm{Cl}(p,q,r)$. Because many of the autodiff primitives require a nondegenerate metric, we will oftentimes assume that $r=0$ to satisfy this condition. Below, we give an example where we first create the spacetime algebra $\Lambda(\mathbb{R}^{1,3})$, and then create some multivectors and perform some basic operations on them:
\begin{lstlisting}
from ga_core.ga import Algebra
st = Algebra(1, 3, 0, ['e0', 'e1', 'e2', 'e3'])  # Cl(1,3): e0^2 = +1, e_i^2 = -1
a = -st.e0 + st.e1 + 2*st.e2  
b = st.e2 - st.e3
a*b            # -2 - e02 + e03 + e12 - e13 - 2 e23   (geometric product)
(a*b).grade(0) # -2                                    (<ab>_0, the inner product)
a ^ b          # -e02 + e03 + e12 - e13 - 2 e23        (outer product)
(a ^ b) * st.I # 2 e01 - e02 - e03 + e12 + e13         (dual w.r.t. I = e0123)
~(a ^ b)       # e02 - e03 - e12 + e13 + 2 e23         (reverse)
\end{lstlisting}
The \texttt{Algebra} class stores all the operators on $\Lambda(V)$ as sparse tensors. Thus, the above operations are all implemented as sparse tensor contractions, and thus run efficiently on GPU and are natively differentiable.

Just as JAX provides mapping utilities to make extending a function to a batch of inputs easy, we write our \texttt{MVArray} as a JAX pytree so that it natively pairs with vmap:
\begin{lstlisting}
def f(x):
    return x*x

xs = st.array([
    1 + st.e01,          # scalar + 2-vector
    st.e0 + st.e12,      # 1-vector + 2-vector
    st.e01 + st.e23,     # non-simple 2-vector (not a blade)
])
xs.shape                 # (3,)
f(xs[0])                 # 2 + 2 e01
jax.vmap(f)(xs)          # [2 + 2 e01, 2 e012, 2 e0123]
\end{lstlisting}

We also provide general helper functions for the constructions used in the experiments. Given \texttt{Algebra}(p,q,0), we can conformalize it to \texttt{Algebra}(p+1,q+1,0). For example, for the experiments in Section~\ref{sec:rotoropt}, we use conformalized $\Lambda(\mathbb{R}^n)$, $\Lambda(\mathbb{R}^{n+1,1})$. This class provides utilities like conformal embedding for points, and ways of creating various translation and rotation rotors. (Note that translations cannot be expressed as rotors in the base algebra $\Lambda(\mathbb{R}^n)$, only rotations can). Other geometric primitives are also easily accessible in this class: for example spheres in the conformal euclidean case and lightcones in conformal spacetime. For further exposition, see \citet{Doran_Lasenby_2003, dorst2007geometric}.

For all \texttt{Algebra} classes, we also provide utilities for easy manipulation of rotors, general exponential and log maps (not closed form), Cayley transforms, along with the trivialization toolbox to easily map (co)tangents between tangent spaces.

\subsection{The Multivector Derivative as a Function Transformation}
\label{sec:api_derivative}

Lemma~\ref{lem:mv-jvp} gives us a recipe to implement the multivector derivative in a few lines of code within our library. For a function which takes in a single multivector and outputs a single multivector, the general multivector derivative can be implemented as:
\begin{lstlisting}
def mv_deriv_fwd(f, input_blade_idxs=None):
    """ partial_I f(X) = sum_{J in I} b^J * jvp(f, X, b_J) """
    def df(X):
        alg = X.algebra
        idxs = input_blade_idxs if input_blade_idxs is 
                    not None else arange(2 ** alg.n) 
        b_lo = alg.basis_blades[idxs]         # MVArray [b_J]
        b_up = alg.recip_basis_blades[idxes]  # MVArray [b^J]
        jvp  = lambda A: jax.jvp(f, (X,), (A,))[1]   # dF_X[A]
        # Note below that '*' is a broadcasted geometric product!
        return (b_up * jax.vmap(jvp)(b_lo)).sum() 
    return df

def mv_deriv_rev(f, output_blade_idxs=None):
    """ partial_O f(X) = sum_{K in O} vjp_<,>(f, X, b^K) * b_K """
    def df(X):
        alg  = X.alg
        idxes = input_blade_idxs if input_blade_idxs is 
                    not None else arange(2 ** alg.n)
        b_lo = alg.basis_blades[idxs]         # MVArray [b_J]
        b_up = alg.recip_basis_blades[idxes]  # MVArray [b^J]
        _, f_vjp = jax.vjp(f, X)
        # metric-aware vjp: vjp(f, X, C^flat)^sharp = dF_X^*[C]
        adj = lambda C: alg.sharp(f_vjp(alg.flat(C))[0])
        # Note below that '*' is a broadcasted geometric product!
        return (jax.vmap(adj)(b_up) * b_lo).sum()
    return df
\end{lstlisting}
We provide two methods, \texttt{mv\_deriv\_fwd} and \texttt{mv\_deriv\_rev} which generalize this to functions which may have arbitrary pytrees of multivectors as inputs and ouptuts. Both methods are extensionally equal in the value they compute, but do so via forward and reverse mode AD, respectively. As mentioned in the text, you can choose which one to use depending on the context of the derivatives you are computing.

\subsection{Vector Manifolds and their Transformations}
\label{sec:api_manifolds}

Much of this work revolves around doing calculations on \texttt{VectorManifold} objects. This enables us to easily compute differential operators on the manifold, and with some mild additional assumptions, a way to also perform Monte Carlo integration on these manifolds. Below, we give the API for this class to demonstrate the capabilities of this class, where we have simplified what is written down from the actual code for the purposes of pedagogy. We emphasize the relative simplicity of the library juxtaposed with the rich, varied set of applications enabled by this framework. The derivative functions simply call \texttt{mv\_deriv\_fwd} with different arguments, and everything else are mere compositions of the projection operator \texttt{proj} and these derivatives.

We stress that geometric quantities like curvature information about the manifold come directly downstream from the shape operator $S(A) = \dot{\partial}_M \dot{P}(A)$. The $\cdot$'s indicate what is being differentiated (in our library we achieve this by applying a \texttt{stop\_grad} to $A$). The curvature tensors are built from its grade-raising part, the \emph{curl tensor} $S_a := \dot\partial_M\wedge\dot P(a)$. For tangent vectors $a$ and $b$, the Riemann curvature is a bivector-valued function of the tangent bivector $a\wedge b$, from which one can compute the Ricci tensor and the scalar curvature:
\begin{equation}
\begin{gathered}
  R(a\wedge b) = \sigma\, P\big(S_a\times S_b\big), \qquad
  \mathrm{Ric}(a) = \sigma\, \partial_M\rfloor S_a, \\
  R = \partial_a\cdot \mathrm{Ric}(a) = \msum_i \langle e^i, \mathrm{Ric}(e_i)\rangle,
\end{gathered}
\end{equation}
where $A\times B =\frac12(AB - BA)$ is the commutator, $\{e_i\}$ is any tangent frame, and $\sigma=+1$ for a definite ambient metric and $-1$ for an indefinite one \citep[Eqs.~5-1.5a, 5-1.19, 5-1.21]{Hestenes_Sobczyk_1984}. The sectional curvature of the tangent plane $a\wedge b$ is
\begin{equation} \label{eq:sectional}
  K(a\wedge b) = -\frac{\langle R(a\wedge b)\,\widetilde{a\wedge b}\rangle_0}{\langle (a\wedge b)\,\widetilde{a\wedge b}\rangle_0},
\end{equation}
where the minus sign accounts for a Euclidean 2-blade squaring to $-1$; for a surface, $K = R/2$ is the Gaussian curvature. With these conventions the unit sphere has $K=1$ and $R=2$, and $S^n$, $H^n$, $\mathrm{dS}^n$ and $\mathrm{AdS}^n$ all come out correctly signed. None of these require a chart or Christoffel symbols, and are coherent in any dimension and signature.

\begin{lstlisting}[float=tp]
class VectorManifold:

    def __init__(...):
        # There are various ways to construct a Vector Manifold,
        # but every vector manifold must provide:
        self.alg = ... # The ambient Algebra object 
        self.I_fn = ... # Function M -> unit k blade of self.alg

        # If we can sample points on M (with a density),
        # then we provide a Monte Carlo, multivector valued
        # integration method 
        self.sample_pts = ...  
        self.dX_fn = ... # (scalar valued density)

        # Finally, we might have a bounding VectorManifold
        # to compute, e.g., normal derivatives 
        self.bdry = ... 

    def factorize(self, A):
        """ Given some $k$-field which is everywhere a blade, 
        return a list of vector factors [a1, ..., ak]: """
    

    def proj(self, F):
        """ The projection operator
        Given F : M -> Cl, returns P(F) = (F | I.inverse()) | I"""
        
    def derivative(self, F, cov=True, deriv_type='full'):
        """ The manifold derivative operator
        Given F : M -> Cl, returns the following operators D@F, where:
        - D in [partial_M, nabla_M] # Intrinsic or covariant
        - @ in [*, ^, |] # 'full', 'outer', 'inner'"""
    
    def laplacian(self, F, cov = True):
        """ Second order manifold derivatives
        Given F : M -> Cl, returns one of:
        -nabla^2 F (Hodge Laplacian, grade preserving)
        -partial^2 F"""

    def shape_operator(self, F):
        """ Hestenes's shape operator
        Given F: M -> Cl, S(F) = partial_M P(stop_grad(F))"""

    def riemann_curvature(self, a, b):
        """ Riemann curvature bivector R(a ^ b) = sigma * P(S_a x S_b),
        with which Ric(a) and scalar curvature can be defined."""

    def integrate(self, F, G=None):
        """ Directed Monte Carlo integral Given F, G : M -> Cl, 
        estimates int_M G dX F,  dX = I |dX|"""

    def transform(self, f, f_inv):
        """ Manifold transformation operator
        Given a right invertible f and a section f_inv, 
        we can create a new VectorManifold M'"""
\end{lstlisting}
As an example of how we used this library for the experiments, see these examples which build on the API we provide:
\begin{lstlisting}
# (1) Divergence-free flow on any 2-manifold
v = M.derivative(lambda x: phi_theta(x) * M.I_fn(x), deriv_type='inner')   # nabla_M | (phi I)
M.derivative(v, deriv_type='inner')                                       # = 0 identically

# (2) Hyperbolic plane from a flat disk in Cl(2,1) 
# Geodesic polar coordinates on the hyperboloid H^2 
# in Cl(2,1): apex e_n, e_n^2 = -1
def polar_to_hyperboloid(u):              
    r = norm(u)
    return cosh(r) * e_n + sinhc(r) * u   

def hyperboloid_to_polar(x):
    x_s = x.grade(1) - x_n(x) * e_n 
    s = norm(x_s)                     
    return (arcsinh(s) / s) * x_s    

H2 = disk.transform(polar_to_hyperboloid, hyperboloid_to_polar)

# (3) Differentiable Euler characteristic of a learned surface 
M1  = square.transform(flow_theta, flow_theta_inv)
chi = M1.integrate(lambda x: M1.I_fn(x).inverse(),
                   G=lambda x: M1.scalar_curvature(x) / (4*pi))
jax.grad(loss)(theta)                          # loss built from chi, differentiated end to end
\end{lstlisting}
With the directed measure $dX = I(x)\lvert dX \rvert$ and $dS$ the corresponding measure on $\partial M$ (if it exists), the fundamental theorem of multivector calculus becomes:
\begin{equation}
  \int_M \dot G\,dX\,\dot\partial_M F \;=\; \int_{\partial M} G\,dS\,F,
\end{equation}
for fields $F,G$ \citep{Hestenes_Sobczyk_1984,Doran_Lasenby_2003}. Its grade components recover Stokes', Green's, and the divergence theorem in any dimension. Our toolkit provides a computational way to approximate either side of this equation. In addition to giving you the ability to solve PDEs, this theorem also gives rise to constructive ways to approximate both inverse functions via the inverse function theorem or implicit functions via the implicit function theorem.

\section{Experiment Details}

\subsection{Shared Setup}
\label{sec:shared_setup}

All experiments are implemented in JAX with 64-bit floats. Except for the implicit-surface networks of Appendix~\ref{sec:minsurf_details}, every experiment runs on CPU; wall-clock times quoted below are from an Apple M3 laptop unless stated otherwise. The implicit-surface networks were trained on a cluster with one Nvidia L40 GPU for speed (but the training also could be done on the laptop).

\subsection{Explicit Manifold Learning}
\label{sec:explicit_details}

\subsubsection{Non-Euclidean Tilings}
\label{sec:tilings_details}
\paragraph{Symmetries}
This experiment requires 3 inputs: a manifold (as a \texttt{VectorManifold} object), a symmetry group acting on that manifold, and a target shape. On $S^2$, the group is the icosahedral rotation group $I\cong A_5$ of order 60, generated by rotations of order 5, 3 and 2 about a vertex, a face centre and an edge midpoint of the icosahedron. The initial tiling, the icosidodecahedron (12 pentagons and 20 triangles), is the orbit of one pentagonal and one triangular seed tile under this group. On $H^2$ the group is the $(2,3,7)$ Coxeter group, generated by the reflections in the sides of a hyperbolic triangle with angles $\pi/2$, $\pi/3$ and $\pi/7$. Each reflection is a single vector versor $X\mapsto -rXr^{-1}$ in the conformal algebra $\mathrm{Cl}(3,1)$ (in fact in its $\mathrm{Cl}(2,1)$ subalgebra spanned by $e_1, e_2, e_n$) \citep{Doran_Lasenby_2003, dorst2007geometric}. The rectified $\{3,7\}$ tiling is the orbit of a seed tile under this group. This construction easily generalizes to any hyperbolic $\{p,q\}$ tiling, i.e.\ any $p,q$ with $\frac1p + \frac1q < \frac12$. Because this group is infinite, for the figures we enumerate it breadth-first up to hyperbolic radius $\rho \leq 3.5$, for a total of 304 tiles (208 triangles and 96 heptagons); see Fig.~\ref{fig:tilings_initial}.

\paragraph{Model and Loss}
Recall that we desire a divergence free 1-field on our 2 dimensional manifold $M$, and achieve this via $v_\theta = \nabla_M \rfloor (\phi_\theta I_M)$ for a scalar valued neural network $\phi_\theta$, which guarantees $v_\theta$ is divergence free by construction. This network is a random-Fourier-feature network (64 frequencies giving 128 features, with scale 4.0 on $S^2$ and 2.5 on $H^2$) followed by an MLP of width 64, depth 2 and GELU activations. It depends only on position and not on time, so $v_\theta$ is an autonomous vector field whose time-1 flow is a diffeomorphism. To make the flow equivariant under the tiling's symmetry, we symmetrize the potential,
\begin{equation}
  \Phi_\theta(x) = \frac{1}{|G|}\sum_{g\in G} \mathrm{sgn}(g)\, \phi_\theta(g\cdot x),
\end{equation}
where $\mathrm{sgn}(g) = -1$ for orientation-reversing $g$: a reflection flips $I_M$, so $\Phi_\theta$ must be odd under reflections for $v_\theta$ to be equivariant. On $S^2$, $G$ is the full rotation group. On $H^2$, $G$ is a finite patch of 96 group elements (those mapping the origin into the 8 nearest tiles), so equivariance there is only approximate.

On $H^2$ we work on the hyperboloid in $\mathbb{R}^{2,1}$, parameterized by the geodesic polar coordinates $u\mapsto\cosh|u|\,e_n + \frac{\sinh|u|}{|u|}u$ of Appendix~\ref{sec:api_manifolds}, and display the result on the Poincar\'e disk by stereographic projection. To realize our tiling, we take the boundary of a single seed tile and integrate $v_\theta$ between $t=0$ and $t=1$ acting on these points using a fixed step size RK4 integrator (15 steps, $h=1/15$). On $S^2$ the points are reprojected onto the sphere after every step; on $H^2$ we integrate in the parameter $u$, pushing $v_\theta$ forward by the differential of the inverse parameterization (a single \texttt{jvp}), so no reprojection is needed.

Let $x_1,\ldots,x_T$ be the points on the boundary of the seed tile (7 per edge, so $T=35$ on $S^2$ and $T=49$ on $H^2$), and let $\Phi^\theta$ denote the RK4 flow map. Each $x_i$ is matched to the target point $y_i$ at the same polar angle $\vartheta_i$ about the tile centre. The targets are given by the radius profiles $r(\vartheta) = R_0\big(1 + 0.3\cos(5\vartheta+0.6)\big)$ on $S^2$ (a chiral, $\mathbb{Z}_5$-symmetric pinwheel) and $r(\vartheta) = R_0\big(1 + 0.18\cos 7\vartheta + 0.05\cos 14\vartheta\big)$ on $H^2$ (a $D_7$-symmetric gear), each mapped onto the manifold by the exponential map at the tile centre, with $R_0$ the mean radius of the seed pentagon on $S^2$ and the circumradius of the seed heptagon on $H^2$. We minimize
\begin{equation}
  \mathcal{L}(\theta) = \frac1T\sum_{i=1}^T \big\lVert \Phi^\theta(x_i) - y_i\big\rVert^2,
\end{equation}
where $\lVert\cdot\rVert$ is the Euclidean norm in $\mathbb{R}^3$ on $S^2$ and in Poincar\'e-disk coordinates on $H^2$.

\paragraph{Training \& Metrics}
We optimize the parameters of the neural network with the Adam optimizer with a learning rate of $1.5\cdot10^{-3}$ and gradient clipping at global norm 2, for 80 full-batch steps (seed 42). We achieve final losses of $1.25\cdot 10^{-3}$ on the sphere and $9.49\cdot 10^{-4}$ on the disk. We show how the tiling looks during the integration, along with the training loss curves, in Fig.~\ref{fig:tilings_snapshots}. Training for each took $<1$ hour on the laptop.

\begin{figure}[htbp]
\centering
\includegraphics[width=0.48\textwidth]{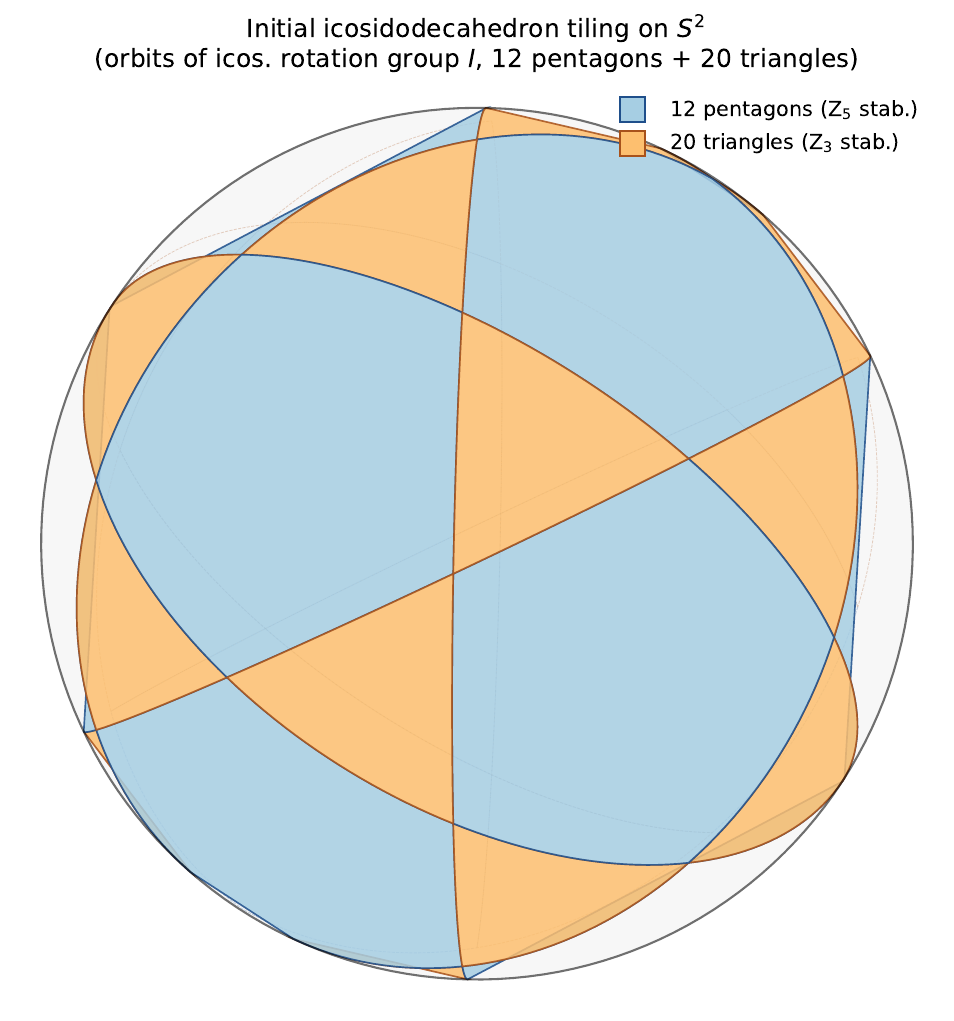}\hfill
\includegraphics[width=0.48\textwidth]{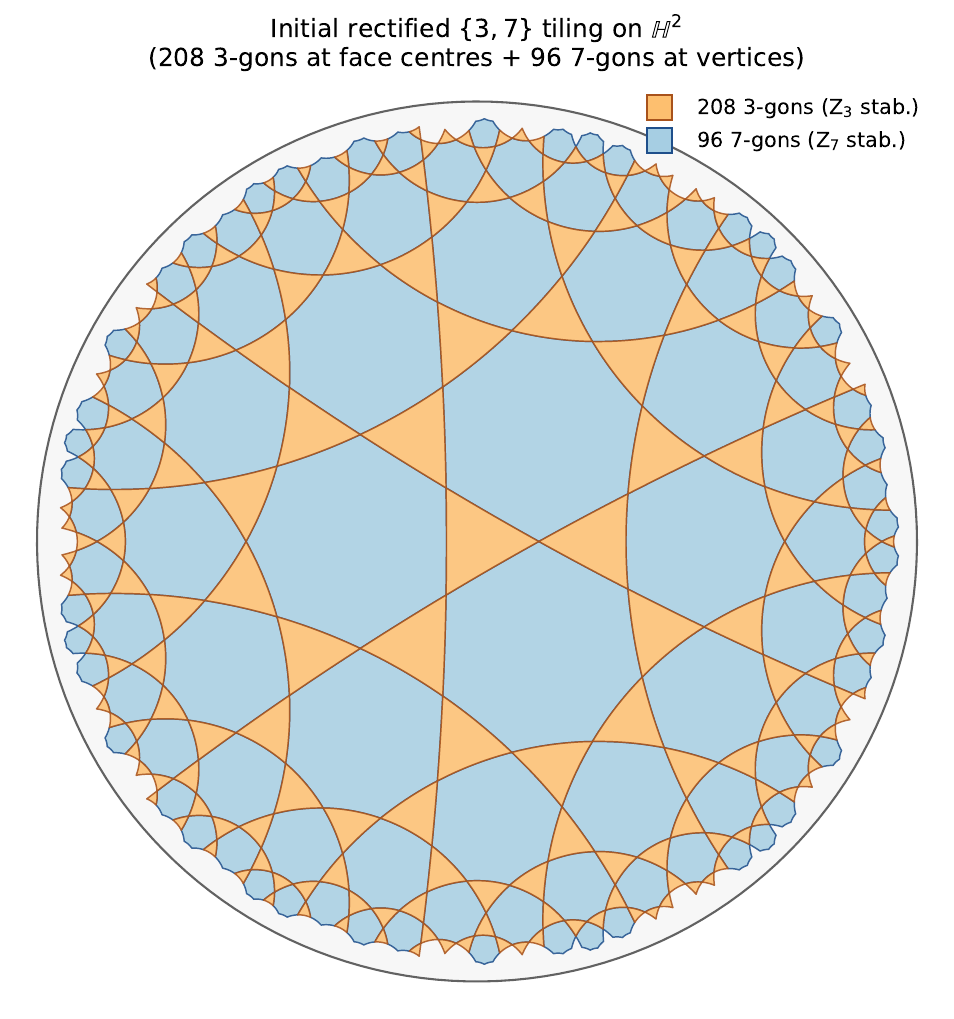}
\caption{Initial tilings. \emph{Left:} the icosidodecahedron on $S^2$. \emph{Right:} the rectified $\{3,7\}$ tiling of the Poincar\'e disk up to hyperbolic radius $3.5$.}
\label{fig:tilings_initial}
\end{figure}

\begin{figure}[htbp]
\centering
\includegraphics[width=0.98\textwidth]{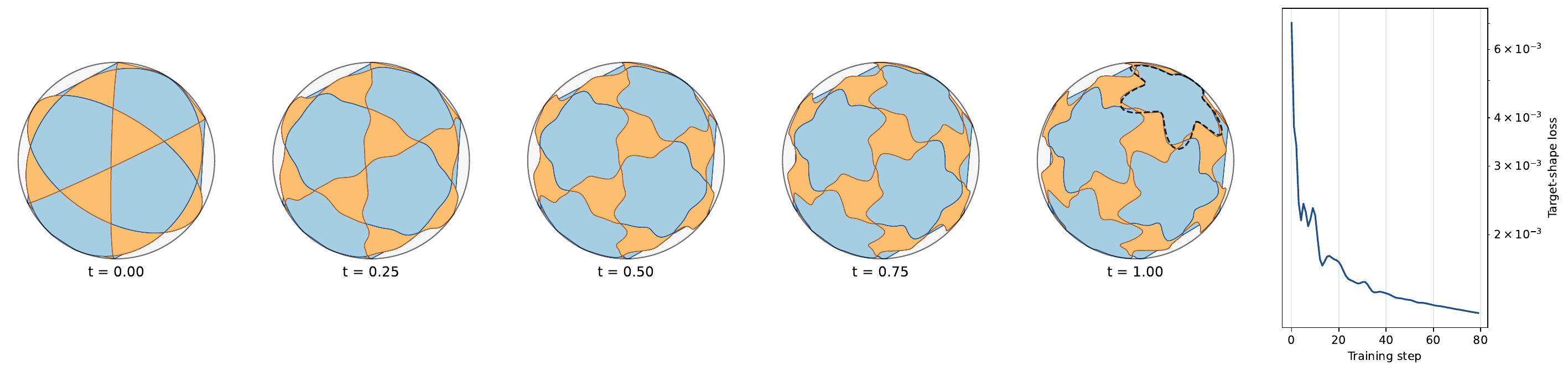}\\[0.3em]
\includegraphics[width=0.98\textwidth]{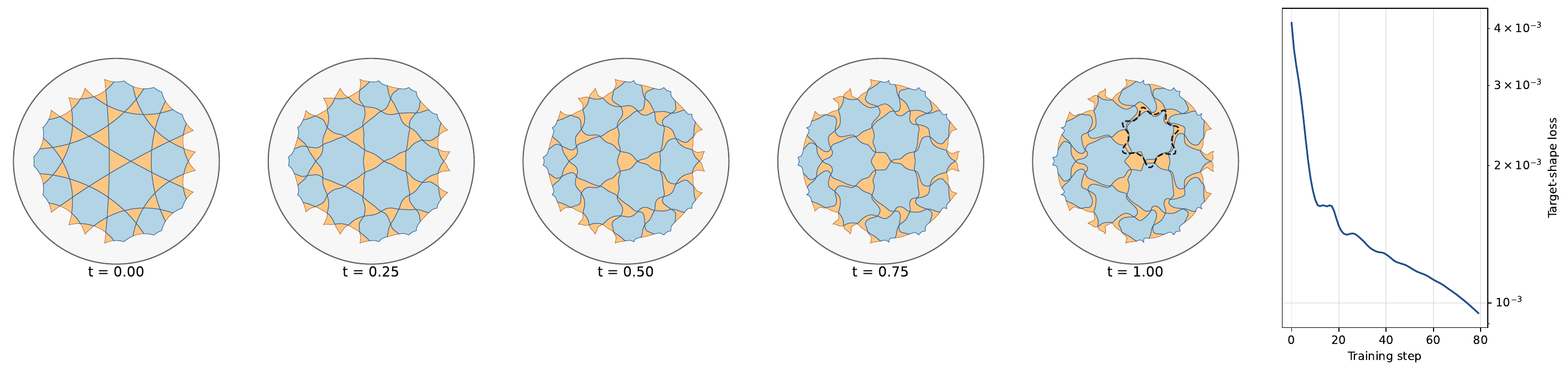}
\caption{The tilings under the learned divergence-free flow at $t = 0, 0.25, 0.5, 0.75, 1$. \emph{Top:} pentagons (blue) deformed into a chiral $\mathbb{Z}_5$ pinwheel on $S^2$, under the icosahedral rotation group. \emph{Bottom:} heptagons (blue) deformed into a $D_7$ gear on $H^2$, under the $(2,3,7)$ Coxeter group. In both, the triangles (orange) flow under the same trained velocity field, and the dashed curve at $t=1$ is the target. The rightmost panels show the training loss (log scale) against the training step.}
\label{fig:tilings_snapshots}
\end{figure}

\subsubsection{Backpropagation Through the Euler Characteristic}
\label{sec:flow_matching_details}

\paragraph{Flow Parameterization}
We parameterize a velocity field $v_\theta(t,x)$ on $\mathbb{R}^3$ as a ConcatSquash MLP \citep{grathwohl2018ffjord} with three layers $3\to96\to96\to3$ and tanh activations, where each layer computes $W_1x\odot\sigma(W_2t) + W_3t$, so that time enters through gating and bias. Integrating $v_\theta$ from $t=0$ to $t=1$ with fixed-step RK4 (4 steps) gives the flow map $\varphi_\theta$, a diffeomorphism of $\mathbb{R}^3$ \citep{chen2018neural}. As a corollary, any manifold $M\subset\mathbb{R}^3$ flowing under $\varphi_\theta$ will be diffeomorphically mapped to a new manifold $M'$; our AD toolbox immediately gives us access to all the differential operators on $M'$. Additionally, if $M$ is compact, then $M'$ will also be and so we can integrate on this manifold. This is what allows us to estimate quantities such as the Euler characteristic of the manifold. For the particular experiment in the main text, we took $M$ to be the square $[-1,1]^2\times\{0\}\subset\mathbb{R}^3$, parameterized by $u\in\mathcal{D}=[-1,1]^2$.

\paragraph{Loss}
For the experiment in the paper, we compute the scalar curvature $R$ (Appendix~\ref{sec:api_manifolds}) at each sampled point of the surface, using the pushed-forward coordinate vectors $r_u = \partial_u\varphi_\theta$, $r_v = \partial_v\varphi_\theta$ as the tangent frame:
\begin{equation}
    R = \partial_a\cdot\mathrm{Ric}(a) = \msum_{i} \langle e^i, \mathrm{Ric}(e_i)\rangle, \qquad \{e_1, e_2\} = \{r_u, r_v\}.
\end{equation}
On a 2-manifold the scalar curvature is simply a rescaling of the Gaussian curvature, $R = 2K$; Gaussian curvature, however, is only defined for 2-manifolds, whereas $R$ is defined in any dimension and signature and our library computes it with no dimension-specific code. We estimate the Euler characteristic from $B$ uniform samples $u_i\in\mathcal{D}$ via Eq.~\ref{eq:gauss_bonnet}:
\begin{equation}
    \hat{\chi}_\theta = \frac{|\mathcal{D}|}{4\pi B}\sum_{i=1}^B R_i\,\lvert\mathrm{d}X\rvert_i, \qquad \lvert\mathrm{d}X\rvert_i = \lvert r_u\wedge r_v\rvert(u_i).
\end{equation}
So, the loss function we minimize, matching the unit sphere's scalar curvature $R_\ast = 2$, is:
\begin{equation}
    \mathcal{L}(\theta) = \frac{\sum_i \lvert\mathrm{d}X\rvert_i\,\big(\tfrac12(R_i - R_\ast)\big)^2}{\sum_i \lvert\mathrm{d}X\rvert_i} + \lambda_{\mathrm{GB}}\big(\hat\chi_\theta/2 - 1\big)^2, \qquad \lambda_{\mathrm{GB}} = 1.5.
\end{equation}
As stated in the text, we optimize this function with and without the Gauss--Bonnet term ($\lambda_{\mathrm{GB}}=0$), and see the difference in the learned manifold $M'$. For these experiments, we train with $B=96$ fresh samples per step and evaluate with $B=4096$.

It may seem paradoxical that $\hat\chi_\theta$ changes during training at all, since $\chi$ should be invariant with respect to any diffeomorphism. The resolution is that Gauss--Bonnet for a surface with boundary reads
\begin{equation}
    \int_M K\,\lvert dX\rvert + \int_{\partial M} k_g\,dS = 2\pi\chi(M)
\end{equation}
with $k_g$ the geodesic curvature of the boundary. Since the square has $\chi = 1$, our estimator satisfies $\hat\chi_\theta = 1 - \frac{1}{2\pi}\int_{\partial M} k_g\,dS$. Driving $\hat\chi_\theta\to 2$ therefore forces $\int_{\partial M}k_g\,dS \to -2\pi$, i.e.\ the boundary must shrink to a small loop that pinches the surface closed, which is exactly the behavior in Fig.~\ref{fig:exp2_euler_v2}.

\paragraph{Training \& Metrics}
We train the parameters $\theta$ via the Adam optimizer with a learning rate of $2\cdot10^{-3}$ and gradient clipping at global norm 5 (seed 0), for 4000 iterations with the Gauss--Bonnet term and 150 iterations without it. We show loss curves, along with the intermediate Euler characteristic estimates, in Figs.~\ref{fig:flow_cap} and~\ref{fig:flow_sphere}. The final Euler characteristic estimate (on $B=4096$ evaluation samples) comes out to $\hat{\chi} \approx 1.97$, with surface area $12.29$ ($98\%$ of $4\pi$) and area-weighted mean Gaussian curvature $1.006$, indicating that the learned $M'$ is pretty close to a fully closed unit sphere, whereas without the Gauss--Bonnet term the square collapses to a small spherical cap with $\hat\chi\approx 0.05$, area $0.34$ and mean curvature $0.98$. On an Apple M3 laptop CPU, training takes 444\,s with the Gauss--Bonnet term (4000 iterations) and 78\,s without it (150 iterations). We also show more intermediate steps of the final learned flows being applied in the top rows of the same figures.

\begin{figure}[htbp]
\centering
\includegraphics[width=0.98\textwidth]{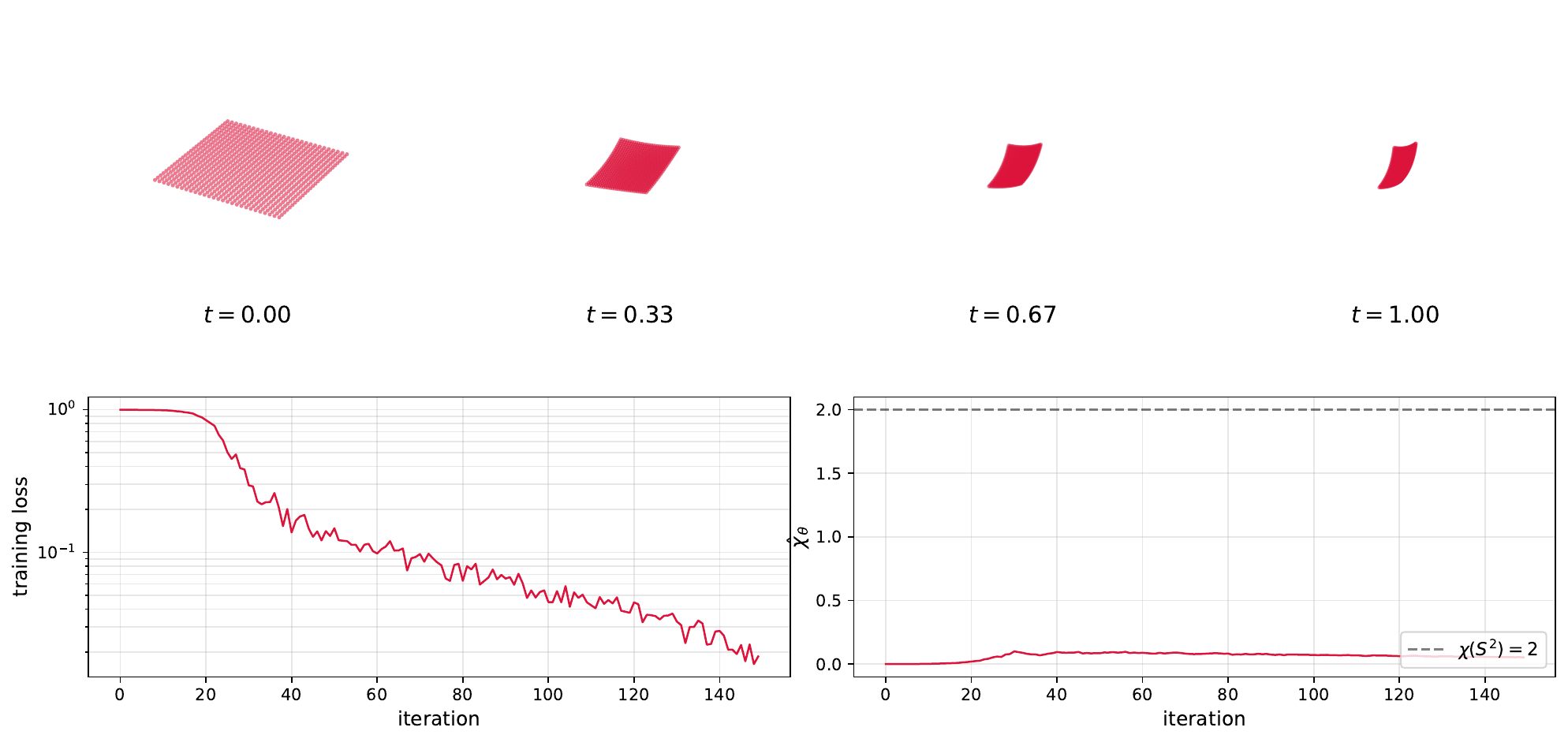}
\caption{The square under the curvature loss alone ($\lambda_{\mathrm{GB}} = 0$). \emph{Top:} the learned flow at $t = 0, \tfrac13, \tfrac23, 1$, which collapses the square into a small spherical cap. \emph{Bottom:} training loss (left) and the Monte Carlo Euler characteristic estimate $\hat\chi_\theta$ (right) over training; without the topological term, $\hat\chi_\theta$ stays far from the sphere's value of 2.}
\label{fig:flow_cap}
\end{figure}

\begin{figure}[htbp]
\centering
\includegraphics[width=0.98\textwidth]{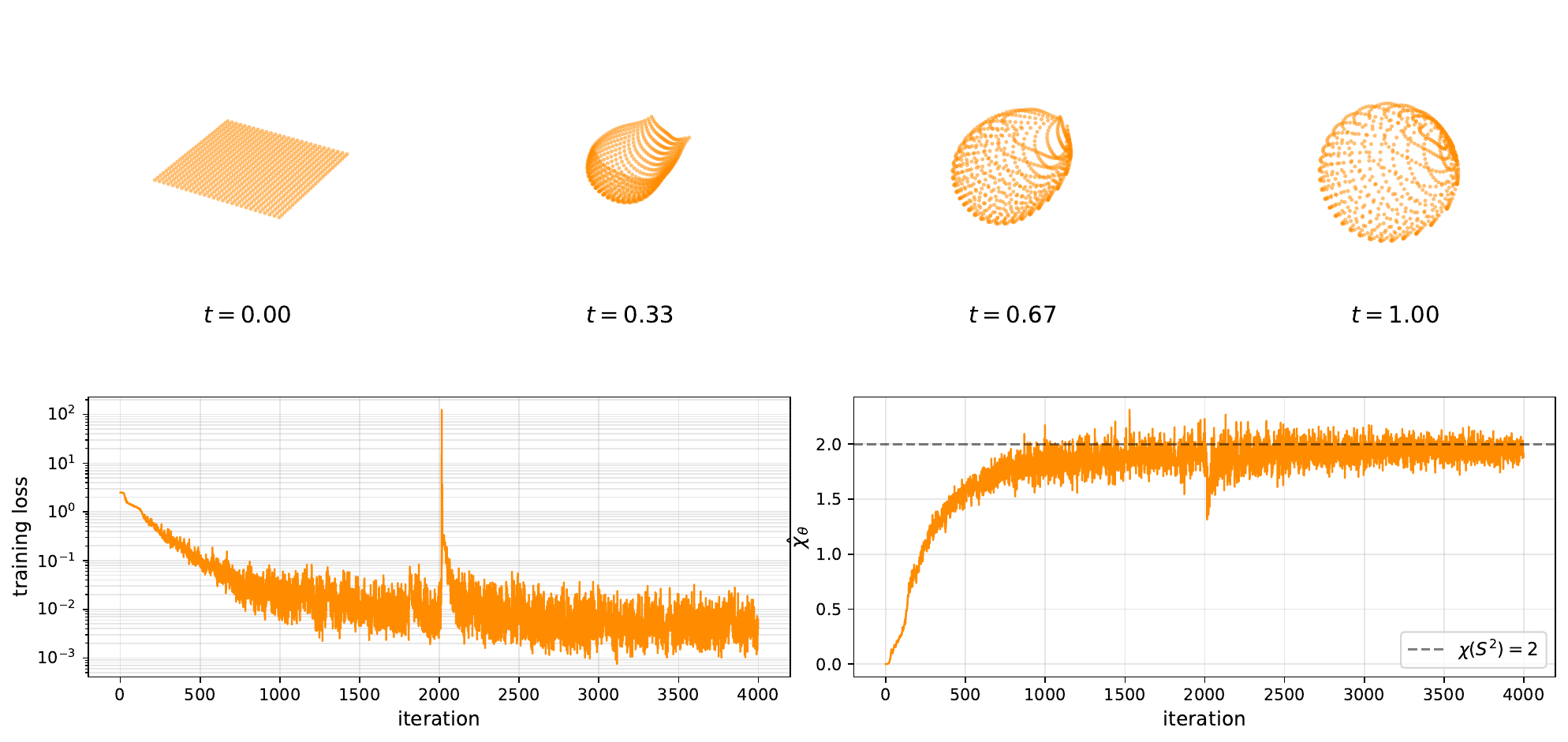}
\caption{The square under the curvature loss with the Gauss--Bonnet term ($\lambda_{\mathrm{GB}} = 1.5$). \emph{Top:} the learned flow at $t = 0, \tfrac13, \tfrac23, 1$, which closes the square into a sphere. \emph{Bottom:} training loss (left) and the Monte Carlo Euler characteristic estimate $\hat\chi_\theta$ (right) over training, which approaches 2.}
\label{fig:flow_sphere}
\end{figure}

\subsection{Implicit Manifold Learning}
\label{sec:implicit_details}

\subsubsection{Minimal Surfaces}
\label{sec:minsurf_details}

\paragraph{Surface Representation}
As stated in the text, the idea of this experiment is to parameterize an $(n-k)$-dimensional surface with boundary $\Gamma$ via the $k$-form $\omega_\theta = \omega_\Gamma + \partial_{\mathcal{I}_1} \wedge \eta_\theta$ where $\omega_\Gamma$ is the Biot-Savart field of Eq.~\ref{eq:biotsavart}. Here $G(x) = |x|^{2-n}/((n-2)S_{n-1})$ is the Green's function of the Laplacian in $\mathbb{R}^n$, with $S_{n-1}$ the area of the unit $(n-1)$-sphere, so $\partial_{\mathcal{I}_1} G(x) = -x/(S_{n-1}|x|^n)$. Computationally, we evaluate the integral on a fixed closed polyline discretization of $\Gamma$, summing $\partial_{\mathcal{I}_1}G(x-m_i)\wedge \Delta\Gamma_i$ over segment midpoints $m_i$ and segment vectors $\Delta\Gamma_i$ and dualizing. We use 120 segments in codimension one; in codimension two we use 480 segments for the torus knot and 320 for the polynomial graph, and regularize the kernel by replacing $|x|^2$ with $|x|^2+\epsilon^2$, $\epsilon = 0.06$. $\eta_\theta$ is a random-Fourier-feature network (Gaussian frequencies, $\sin$ and $\cos$ features) followed by a softplus MLP producing a $(k-1)$-form: in codimension one, a scalar potential with 1024 frequencies at scale 2.0 and an MLP of width 256 and depth 4; in codimension two, a 1-form with 96 frequencies at scale 0.5 and an MLP of width 128 and depth 3 whose last layer is initialized to zero, so that training starts from $\omega_\theta = \omega_\Gamma$. Again using $(\partial_{\mathcal{I}_1} \wedge)^2 = 0$, this parameterization ensures that $\partial_{\mathcal I_1} \wedge \omega_\theta = \delta_\Gamma$, which is equivalent to the condition that the represented surface has boundary $\Gamma$.

\paragraph{The Mass Norm}
As stated in the text, the objective that we minimize is the \emph{mass norm} of $\omega_\theta$. While we gave an informal definition for the mass norm, the formal definition is:
\begin{equation}
    \lVert \omega \rVert := \inf\Big\{ \msum_j \lvert \beta_j\rvert \;:\; \omega = \msum_j \beta_j B_j,\ B_j \text{ unit $k$-blades}\Big\}.
\end{equation}
There is no closed form for this infimum for $3 \le k \le n-3$, so this is why we focus our experiments on the $(n,k) = (3,1)$ and $(4,2)$. For $k=1$ the mass norm is the Euclidean norm. For the $n=4, k=2$ case, there is a closed form \citep{federer1969gmt} that is easily differentiable almost everywhere: splitting $\omega$ into its self-dual and anti-self-dual parts $\omega_\pm = \frac12(\omega \pm \star\omega)$,
\begin{equation}
    \lVert\omega\rVert = \sqrt{2}\,\max\big(\lvert\omega_+\rvert, \lvert\omega_-\rvert\big).
\end{equation}
Minimizing this loss ends up turning $\omega_\theta$ into an approximately simple bivector (one that is a blade). In practice we smooth the pointwise norm as $m_\delta(\omega) = \sqrt{\lVert\omega\rVert^2 + \delta^2}$, and minimize a Monte Carlo estimate of its integral over a box $\Omega\supset\Gamma$,
\begin{equation}
    \mathcal{L}(\theta) = \frac{1}{N}\sum_{i=1}^N \frac{m_\delta\big(\omega_\theta(x_i)\big)}{q(x_i)}, \qquad x_i\sim q,
\end{equation}
with $N = 4096$. For the first $25\%$ of training $q$ is uniform on $\Omega$; afterwards it is a mixture of $30\%$ uniform samples and $70\%$ samples from a lattice with cell probabilities proportional to $|\omega_\theta|$, refreshed every 250 steps. The box is $[-1.5,1.5]^3$ in codimension one and a cube of half-width $\max|\Gamma| + 0.5$ in codimension two. In codimension one $\delta = 10^{-6}$ is fixed; in codimension two $\delta$ is annealed log-linearly from $5\cdot10^{-2}$ to $5\cdot10^{-6}$ over training.

\paragraph{Extraction}
The learned surface is read off from the $\omega_\theta$ alone; the extractor never sees the closed-form surface. In codimension one, $\omega_\theta = \omega_\Gamma + \nabla f$ with $f$ the learned scalar potential, and the surface is the level set of $f$ at its mean value on $\Gamma$, computed by marching cubes on a $192^3$ grid. We keep only the faces where the Biot--Savart field $\omega_\Gamma$ and $\nabla f$ point the same way and the total field is large, i.e.\ where the current actually concentrates, then clean the mesh (largest component, hole filling), smooth it, and snap its rim onto $\Gamma$. In codimension two the surface is the ridge of $|\omega_\theta|$. We seed points on a grid around $\Gamma$ (keeping the top $15\%$ by mass) and on $\Gamma$ itself, and move each by a few Newton steps that maximize $\log|\omega_\theta|$ within the normal plane of $\omega_\theta$ (a simple $\omega_\theta$ determines its own tangent and normal planes). We keep points at a local maximum inside a neighbourhood of the convex hull of $\Gamma$, densify the result by advancing along the tangent planes, and discard points that are not a maximum relative to a ring of radius $2.5\epsilon$ in the normal plane. For rendering only, the points are meshed by local Delaunay triangulation in their tangent planes.

\paragraph{Training \& Metrics}
We use Adam with gradient clipping at global norm 1 and a staircase exponential learning-rate decay. In codimension one we train for $10^5$ steps with learning rate $5\cdot10^{-4}$, decayed by $0.6$ every $10^4$ steps, which takes about 40 minutes per boundary on one GPU. In codimension two we train for $27{,}000$ steps with learning rate $5\cdot10^{-3}$, decayed by $0.7$ every $2{,}700$ steps (seed 0). The codimension-two ground truths are the graph of $w = z^3 - \lambda z$ over $|z|\le\rho$ ($\rho = 0.8$, $\lambda = 0.6$; exact area $\pi\rho^2(1+\lambda^2) + 3\pi\rho^6 = 5.205$) and the branched disk $u\mapsto(ru^2, su^3)$, $|u|\le 1$ ($r = 0.7$, $s = 0.5$; exact area $\pi(2r^2 + 3s^2) = 5.435$), whose boundary is the $(2,3)$ torus knot. Both are holomorphic curves and hence mass-minimizing \citep{federer1969gmt}. Table~\ref{tab:codim2_metrics} compares the extracted point clouds against dense samples of the exact surfaces.

\begin{table}[htbp]
\centering
\caption{Extracted codimension-two surfaces against the exact minimal surfaces. Chamfer is the symmetric mean nearest-neighbour distance; Hausdorff distances are one-sided (extraction to reference); coverage is the fraction of the reference within $0.03$ of the extraction, and spurious the fraction of extracted points farther than $0.03$ from the reference. Area is that of the render mesh.}
\label{tab:codim2_metrics}
\footnotesize
\setlength{\tabcolsep}{4pt}
\begin{tabular}{@{}lcccccc@{}}
\toprule
boundary & Chamfer & Hausdorff (95\%) & Hausdorff (max) & coverage & spurious & area (exact) \\
\midrule
torus knot & 0.0099 & 0.014 & 0.050 & 94.8\% & 0.13\% & 5.53 (5.44) \\
polynomial graph & 0.0150 & 0.028 & 0.046 & 99.96\% & 1.51\% & 5.12 (5.21) \\
graph $w = z^2$ & 0.0123 & 0.023 & 0.041 & 100.0\% & 0.42\% & 4.69 (4.58) \\
annulus $zw = c$ & 0.0083 & 0.016 & 0.034 & 100.0\% & 0.08\% & 3.24 (3.08) \\
two disks $w = (c/\rho^2)\bar z$ & 0.0126 & 0.021 & 0.048 & 97.4\% & 1.04\% & 5.80 (5.75) \\
\bottomrule
\end{tabular}
\end{table}

In addition to the two surfaces shown in the main text, we tested three further boundaries, with results in Table~\ref{tab:codim2_metrics} and Fig.~\ref{fig:codim2_families}: the holomorphic graph $w = z^2$ over $|z|\le 0.8$ (exact area $4.584$); two circles $|z| = r$, $|z| = R$ on the curve $zw = c$ ($r = 0.55$, $R = 1$, $c = 0.35$), which with opposite orientations bound the holomorphic annulus $\{zw = c\}$ (area $3.079$) and with equal orientations bound the pair of anti-holomorphic disks $w = (c/\rho^2)\bar z$ (area $5.749$). In every case the extraction matches the exact surface closely. Fig.~\ref{fig:codim2_main_angles} shows the two surfaces of the main text at more projection angles.

\begin{figure}[htbp]
\centering
\includegraphics[width=0.84\textwidth]{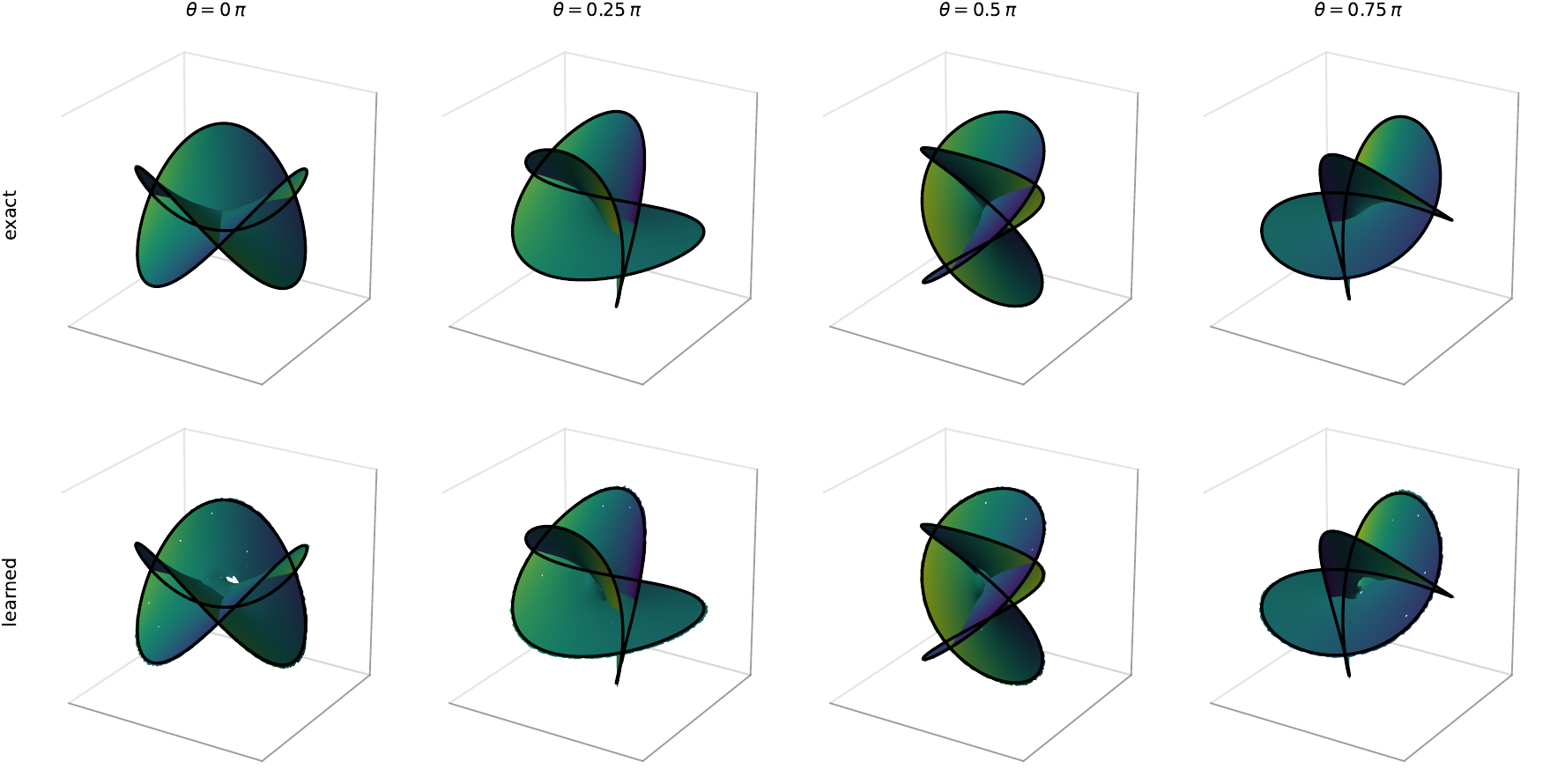}\\
{\small branched disk bounding the $(2,3)$ torus knot}\\[1em]
\includegraphics[width=0.84\textwidth]{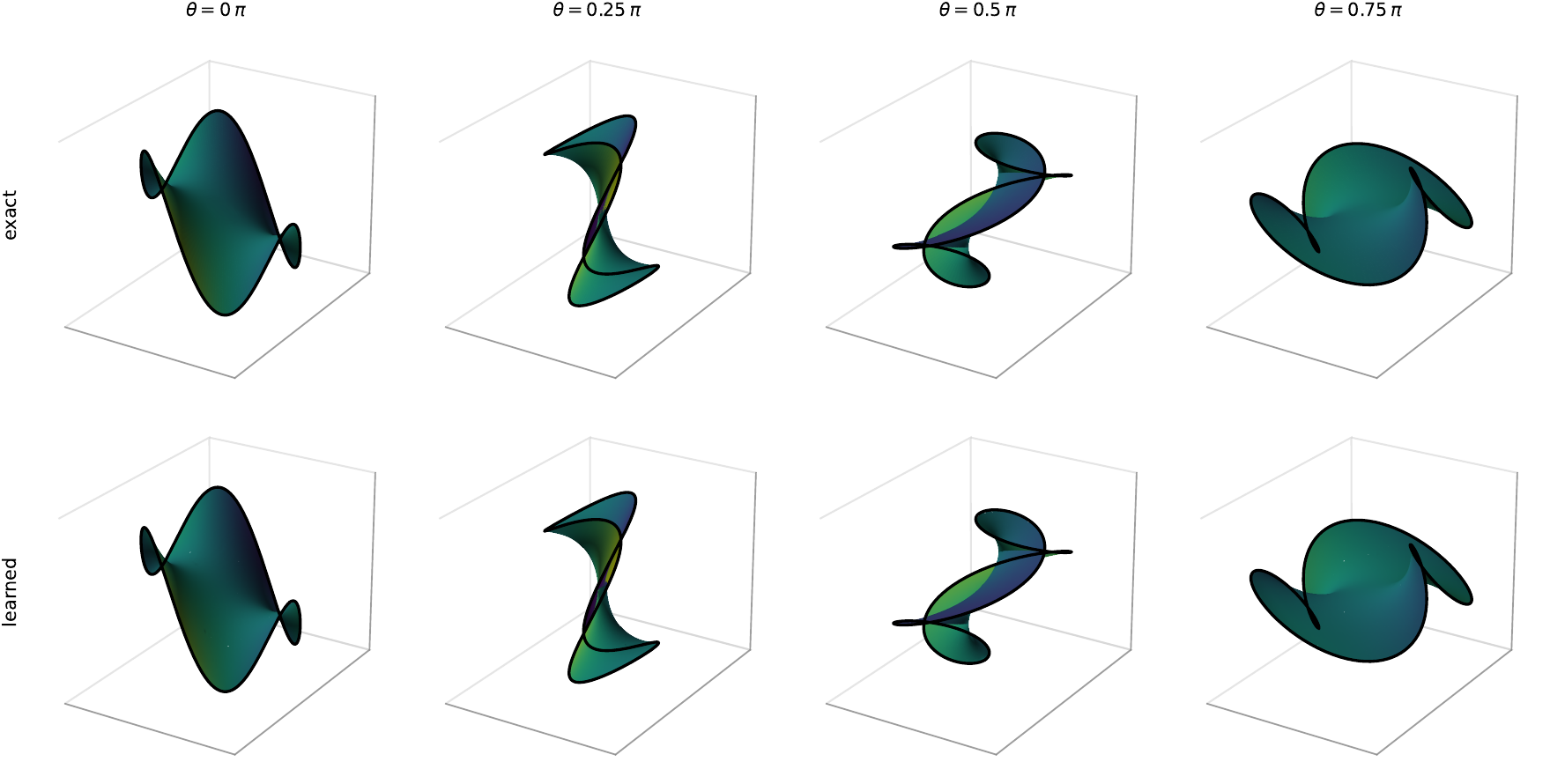}\\
{\small graph $w = z^3 - \lambda z$}
\caption{The two codimension-two surfaces of Figs.~\ref{fig:exp3_codim2_v2} and~\ref{fig:exp3_codim2_torus_v2} at four projection angles $\theta = 0, \pi/4, \pi/2, 3\pi/4$. Each block shows the exact minimal surface (top row) and the surface extracted from the learned field (bottom row).}
\label{fig:codim2_main_angles}
\end{figure}

\begin{figure}[htbp]
\centering
\includegraphics[width=0.84\textwidth]{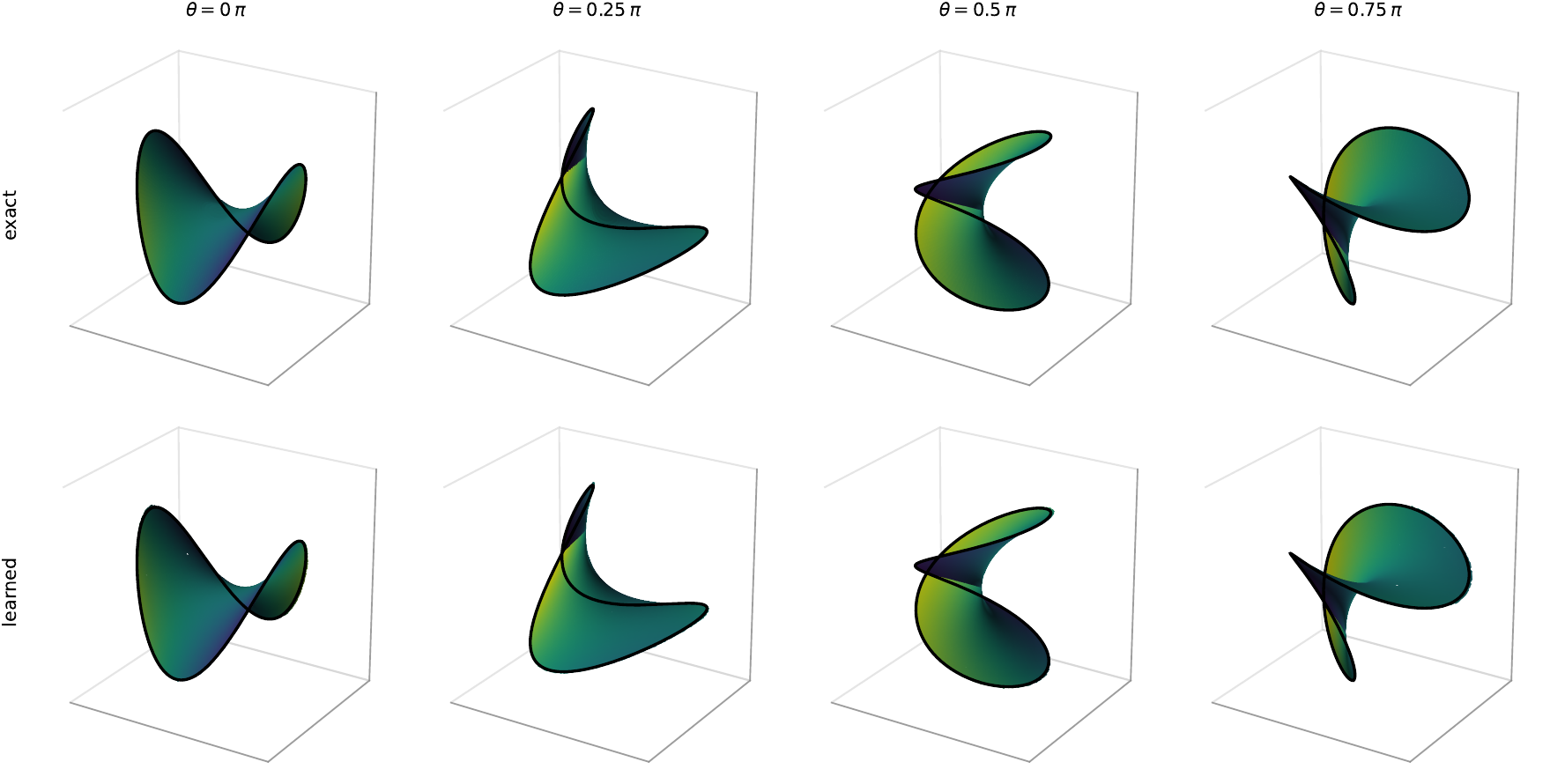}\\
{\small graph $w = z^2$}\\[1em]
\includegraphics[width=0.84\textwidth]{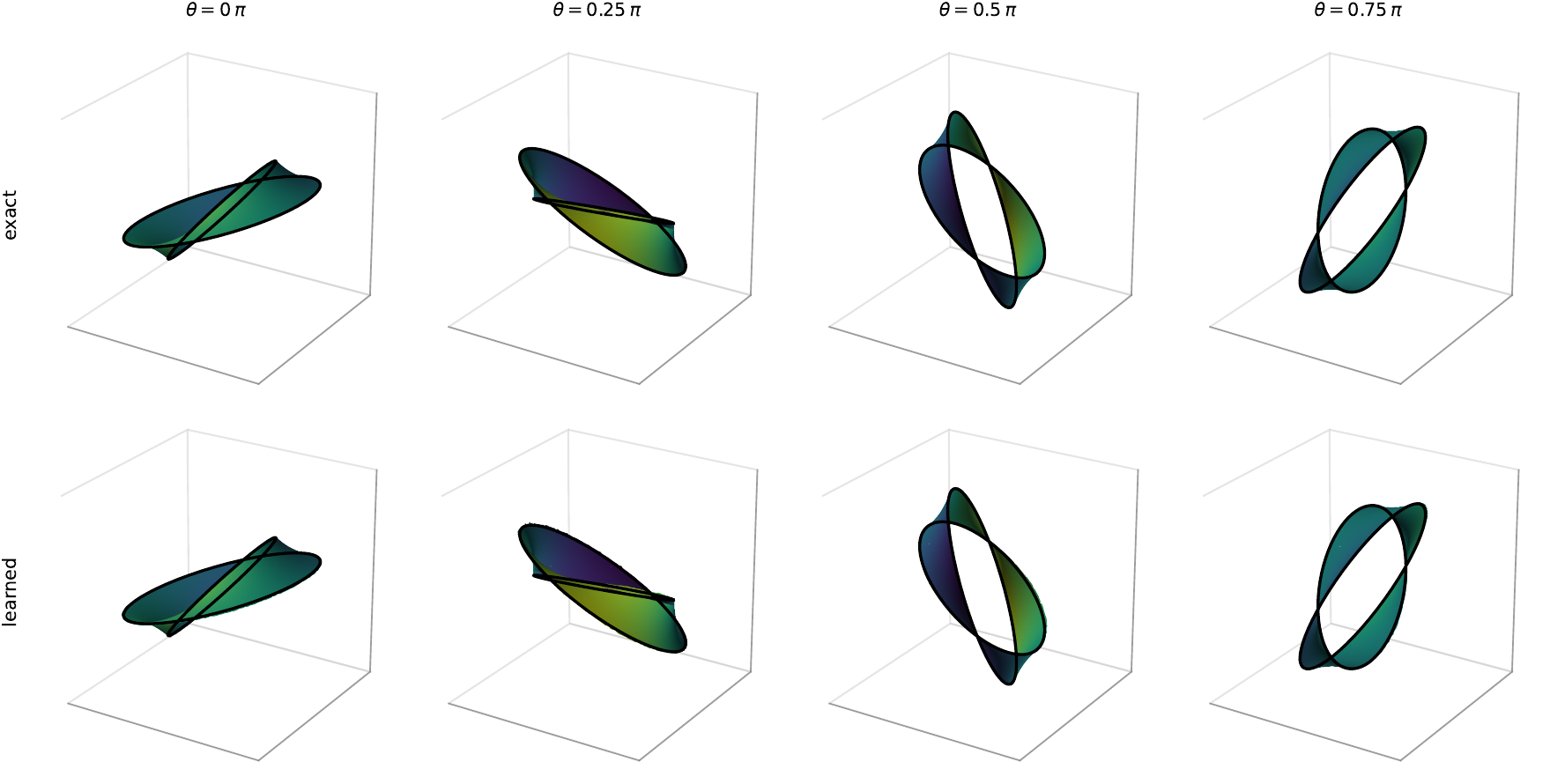}\\
{\small annulus $zw = c$}\\[1em]
\includegraphics[width=0.84\textwidth]{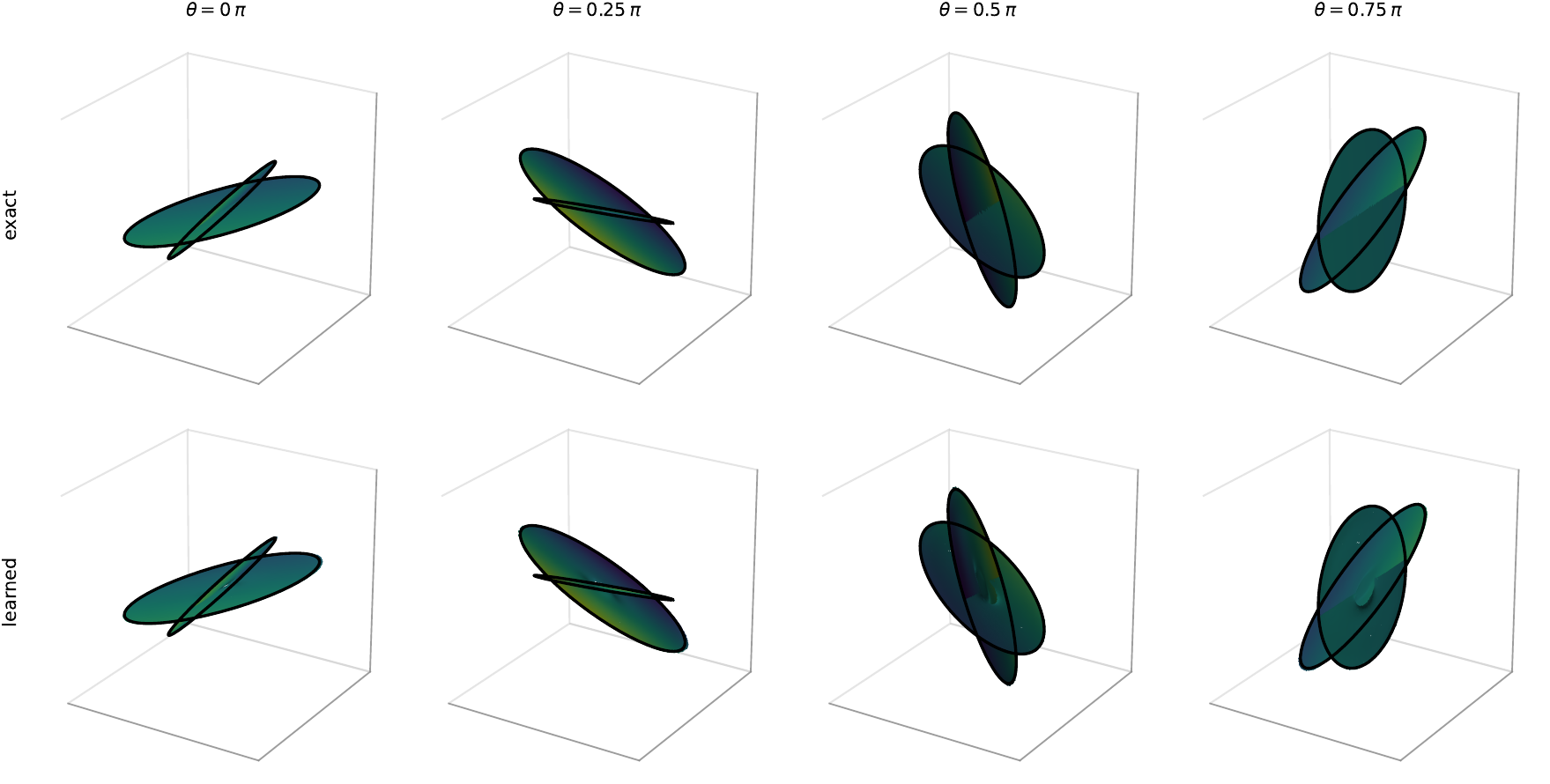}\\
{\small two disks $w = (c/\rho^2)\bar z$}
\caption{The three further codimension-two boundaries of Table~\ref{tab:codim2_metrics} at four projection angles $\theta = 0, \pi/4, \pi/2, 3\pi/4$. Each block shows the exact minimal surface (top row) and the surface extracted from the learned field (bottom row), rendered as in Fig.~\ref{fig:codim2_main_angles}.}
\label{fig:codim2_families}
\end{figure}

\subsection{Inverse Design in Spacetime}
\label{sec:em_details}
\begin{figure}[p]
\centering
\begin{subfigure}{\textwidth}
\centering
\includegraphics[width=0.64\textwidth]{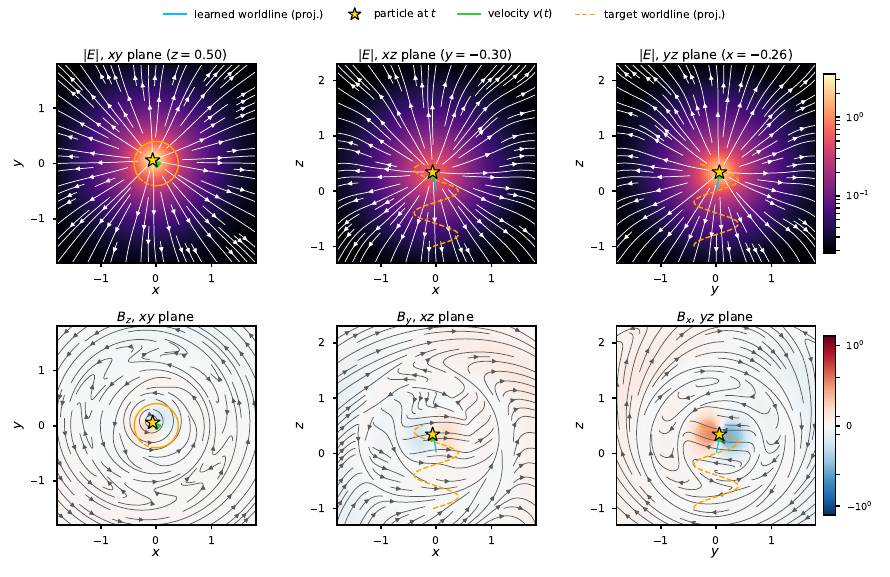}
\caption{After 50 training steps.}
\label{fig:em_stage_early}
\end{subfigure}

\vspace{0.6em}
\begin{subfigure}{\textwidth}
\centering
\includegraphics[width=0.64\textwidth]{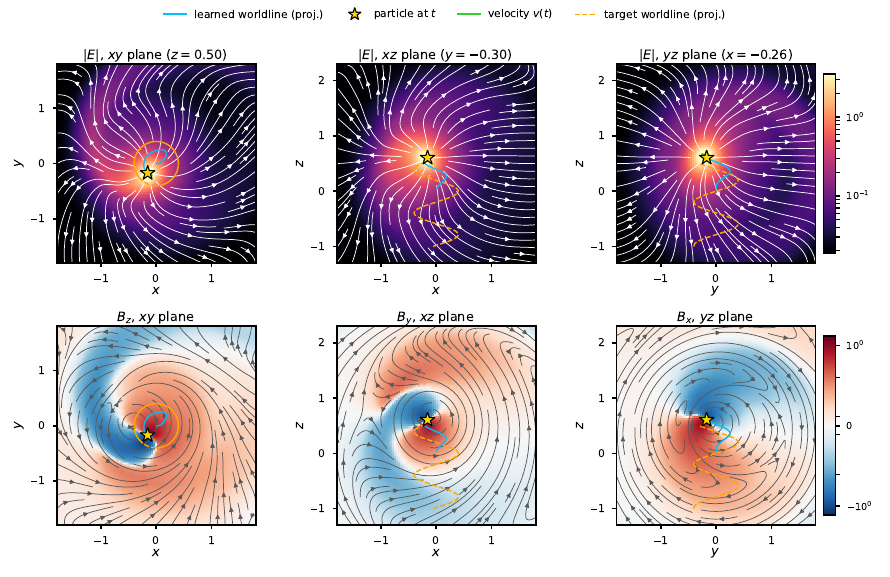}
\caption{After 70 training steps.}
\label{fig:em_stage_mid}
\end{subfigure}

\vspace{0.6em}
\begin{subfigure}{\textwidth}
\centering
\includegraphics[width=0.64\textwidth]{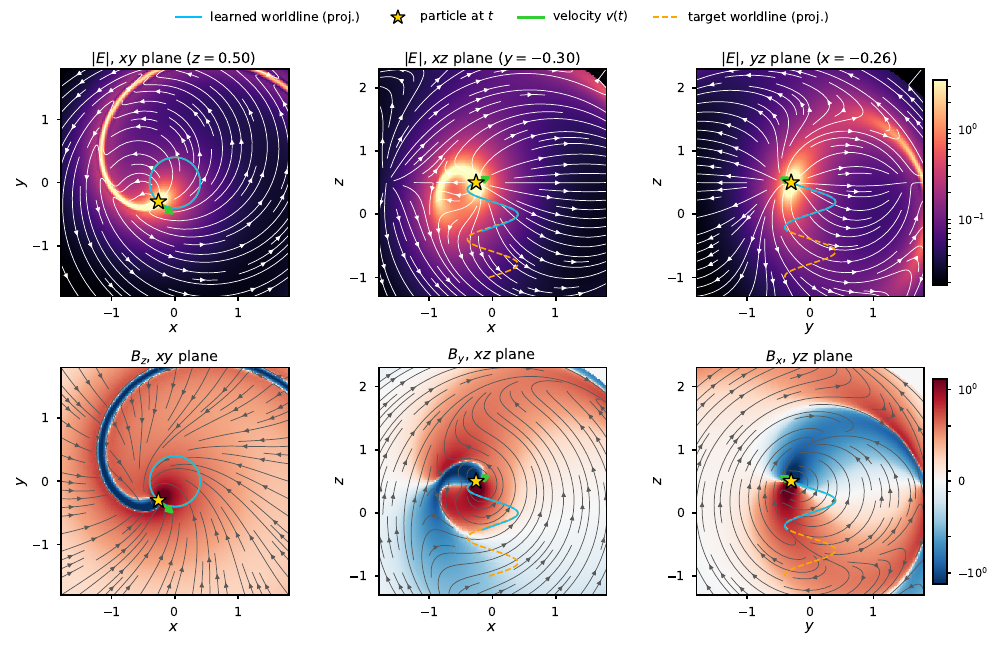}
\caption{Final learned worldline, which reproduces the field of the target (Fig.~\ref{fig:em_fwd_v2}).}
\label{fig:em_stage_final}
\end{subfigure}
\caption{Field of the learned worldline during training, at $t = 2$: $|E|$ with in-plane $E$ streamlines (top) and the out-of-plane component of $B$ with in-plane $B$ streamlines (bottom), on the $xy$, $xz$ and $yz$ planes through the target particle. The learned worldline is drawn solid and the target dashed.}
\label{fig:em_stages}
\end{figure}

\paragraph{Forward Simulator}
We work in the spacetime algebra $\mathrm{Cl}(1,3)$ with basis $e_0, e_1, e_2, e_3$ with $e_0^2 = 1$, $e_1^2=e_2^2=e_3^2=-1$. We parameterize the particle's spatial velocity $v_\theta(t)$ in lab time $t$ via an MLP (width 64, depth 3, GELU activations) of the Fourier features $(t, \sin f_k t, \cos f_k t)$ with 8 frequencies $f_k$ evenly spaced in $[0.5, 6]$, and constrain the raw output $w$ to $v = w / \sqrt{1+\lvert w \rvert^2}$, which enforces $\lvert v \rvert < 1$. The reason for this is that we want the resulting worldline $x_\theta(t) = t\,e_0 + r_\theta(t)$ to be \emph{timelike} (slower than speed of light), which any particle with mass should obey. The spatial position $r_\theta(t)$ is the trapezoidal integral of $v_\theta$ on 341 knots over $t\in[-1,7]$, with cubic Hermite interpolation between knots so that $\dot r_\theta = v_\theta$ exactly, and the particle's 4-velocity is proportional to $e_0 + v_\theta$. Then, suppose we observe something at spacetime location $y$. We desire to find the \emph{retarded time} $\tau_r$, the time at which the particle emitted the signal that the observer at $y$ receives, i.e.\ the root of $(y - x_\theta(\tau_r))^2 = 0$ on the past light cone of $y$. This boils down to a 1 dimensional root finding problem, which we solve with a differentiable bisection search on $[-50, t_y - 10^{-3}]$ to a tolerance of $10^{-8}$ using Optimistix \citep{rader2024optimistix}. The root is differentiated with respect to the trajectory via the implicit function theorem \citep{griewank2008evaluating, blondel2022efficient}. This allows us to compute the Li\'enard--Wiechert potential
\begin{equation}
    A(y) = \frac{q}{4\pi} \frac{u(\tau_r)}{\big\lvert\langle y-x_\theta(\tau_r), u(\tau_r) \rangle\big\rvert}, \qquad u = e_0 + v_\theta,
\end{equation}
with $q = 1$; the normalization of $u$ cancels between numerator and denominator. From this, the Faraday bivector is the single call $F = \partial_{\mathcal{I}_1} \wedge A$. Using the decomposition $F= E + IB$ with pseudoscalar $I = e_{0123}$, where $E = \frac12(F - e_0Fe_0)$ and $B = \frac12 I^{-1}(F + e_0Fe_0)$, we get a recipe to read off the $E$ and $B$ fields in the $\sigma_k = e_k e_0$ frame, which is how we generate the streamline plots in Fig.~\ref{fig:em_fwd_v2}. The final loss is a relative mean squared error between the predicted and observed $F$ over the $N$ observation events,
\begin{equation}
    \mathcal{L}(\theta) = \frac{1}{N}\sum_{i=1}^N \frac{\big\lVert F_\theta(y_i) - F_\ast(y_i)\big\rVert^2}{\big\lVert F_\ast(y_i)\big\rVert^2},
\end{equation}
where $\lVert\cdot\rVert$ is the Euclidean norm on the bivector coefficients. We normalize per event because $\lVert F\rVert^2 \sim |y - x|^{-4}$ spans more than four orders of magnitude across the observation events (a factor of about 200 between the 5th and 95th percentiles).

\paragraph{Training \& Metrics}
The target trajectory worldline is the helix $r_\ast(t)=(0.4\cos 2t,\allowbreak 0.4 \sin 2t,\allowbreak 0.25t )$ in lab time $t$, and we reconstruct it over the window $t\in[1,4]$. To provide data for the model, we evaluate the exact field $F_\ast$ on a spatial grid of $5\times5\times4$ points covering $[-1.8,1.8]\times[-1.8,1.8]\times[-0.4,1.6]$ at 21 evenly spaced times in $[1,6.5]$, excluding events within a radius 1 of the true worldline where the potential is singular; this leaves 1944 noise-free training events. The 20 midpoint time slices (1848 events) are held out for model selection. The observation times extend past the reconstruction window because an observation at time $t$ and distance $d\ge1$ from the particle constrains the worldline only at the retarded time $t - d$. We optimize the trajectory with Adam for 7000 full-batch steps, with gradient clipping at global norm 1 and a learning rate warmed up over 100 steps to $3\cdot10^{-3}$ and then cosine-decayed to $2\cdot10^{-5}$ (seed 1), keeping the checkpoint with the lowest held-out loss.

The recovered worldline has a mean position error of $6\cdot10^{-4}$ and a maximum of $1.2\cdot10^{-3}$ over $t\in[1,4]$, i.e.\ $0.2\%$ and $0.3\%$ of the orbit radius. Training takes 6 minutes on an Apple M3 laptop CPU. As a check of the forward model, the vacuum Maxwell residual $\partial_{\mathcal{I}_1}\cdot F$ at 16 random events away from the worldline has maximum $1.8\cdot10^{-7}$ and mean $1.6\cdot10^{-8}$, against field magnitudes of order one, consistent with the root-finder tolerance.

Causal coverage matters: training the same recipe on observation slices confined to the reconstruction window $[1,4]$ (12 slices) recovers the early worldline (mean error $1.7\cdot10^{-3}$ on $[1,2]$) but not the tail, which no observation can see: on $[3,4]$ the mean error is $4.6\cdot10^{-2}$ (maximum $0.25$), against $5\cdot10^{-4}$ (maximum $8\cdot10^{-4}$) with the extended observation window.

Fig.~\ref{fig:em_stages} shows how the field of the learned worldline evolves during training, at $t = 2$ and on the same three planes through the target particle as Fig.~\ref{fig:em_fwd_v2}, with shared colour scales. After 50 steps (mean position error $0.48$) the learned particle is nearly at rest, so $E$ is almost radial and $B$ is weak. After 70 steps (error $0.22$) the worldline has picked up the helical motion and the field begins to wind into a spiral. The final model reproduces the target field, including the sharp spiral wavefront emitted by the orbiting charge.

\subsection{Optimization on Rotor Groups}
\label{sec:rotor_details}

\paragraph{Multiplicative Descent}
Let $G\subset\mathrm{Spin}(p,q)^+$ be a rotor group whose Lie algebra $\mathfrak{g}$ is spanned by bivectors $\{B_k\}_{k=1}^{\dim\mathfrak{g}}$, with reciprocal basis $\{B^k\}$ in $\mathfrak{g}$. The left-trivialized differential of $f$ at $g$ is the covector $df_e(g)[B] = df_g[gB]$ on $\mathfrak{g}$. Writing $\partial f(g) := \partial_{\mathcal{I}_{\mathrm{even}}} f(g)$ and using that the scalar part is cyclic, $df_g[gB] = \langle \partial f(g), gB\rangle = \langle \partial f(g)\, g, B\rangle$, so its coordinates in the basis $\{B^k\}$ are
\begin{equation} \label{eq:lie_coords}
    c_k(g) = df_e(g)[B_k] = \big\langle \partial f(g)\, g,\; B_k\big\rangle, \qquad
    df_e(g) = \msum_k c_k(g)\, B^k = P_{\mathfrak{g}^*}\big(\partial f(g)\, g\big).
\end{equation}
This is Eq.~\ref{eq:lie_deriv} of the main text, where the restriction to $\mathfrak{g}$ is made explicit by $P_{\mathfrak{g}^*}$: $\partial f(g)\,g$ is a general even multivector, and only its pairing with the generators enters the update. For a scalar loss, $\partial f(g)$ is a single \texttt{vjp} (Lemma~\ref{lem:mv-jvp}). In practice we feed the coordinates $c_k$ to Adam, whose step $u_t\in\mathfrak{g}$ is exponentiated, $g_{t+1} = g_t\exp(u_t)$. This ensures every iterate is a valid rotor. The exponential is the algebra's general multivector exponential; in $\mathrm{Cl}(1,3)$ we evaluate it as the matrix exponential of the $16\times16$ left-multiplication operator $X\mapsto BX$, which JAX differentiates exactly. A subgroup of $G$ spanned by a subset of the generators is selected by a $0/1$ mask on the coordinates $c_k$, so several subgroups of the same algebra can be optimized in a single vectorized call.

\paragraph{Choosing a Positive Definite Metric}
As discussed in the main text, $\sharp:\mathfrak{g}^*\to\mathfrak{g}$ must come from a positive-definite inner product on $\mathfrak{g}$ for $-\eta\,df_e^\sharp$ to be a descent direction, since then $df_e[-\eta\,df_e^\sharp] = -\eta\lVert df_e\rVert^2 < 0$. The multivector inner product $\langle A, B\rangle = \langle AB\rangle_0$ fails this in every case we consider. The metric on $\mathfrak{spin}(n)$ it is negative definite, since Euclidean bivector blades square to $-1$, so the ``natural'' step would be an ascent step. On $\mathfrak{spin}(1,3)$ it is indefinite (boosts $e_0e_i$ square to $+1$, rotations $e_ie_j$ to $-1$), and on the translation subalgebra of $\mathrm{CGA}(3)$ it vanishes identically, since the generators $e_i\wedge e_\infty$ are null and mutually orthogonal. Nor is there a canonical alternative: non-compact groups admit no bi-invariant metric, so some non-invariant choice is unavoidable. We declare the chosen generators orthonormal,
\begin{equation}
    \langle B_j, B_k\rangle_\sharp := \delta_{jk}, \qquad
    df_e(g) = \msum_k c_k B^k \;\longmapsto\; df_e(g)^\sharp = \msum_k c_k B_k,
\end{equation}
which is the identity on coordinates. This coincides with the choice a coordinate-array AD system makes implicitly, but here it is an explicit, replaceable ingredient.

\paragraph{Chart-based baselines}
We compare against two updates that keep chart coordinates $b\in\mathbb{R}^{\dim\mathfrak{g}}$, $g = \exp(\sum_k b_kB_k)$, as the optimizer state. The \emph{static chart} (static trivialization, the recipe of \citet{lezcano2019trivializations}) runs Adam on $b$ directly. Its gradient is not $df_e$: with $T_B := \exp(-B)\,d\exp_B = \varphi(-\mathrm{ad}_B)$, $\varphi(x) = (e^x-1)/x$ and $\mathrm{ad}_B C = BC - CB$, the chart gradient is $T_B^*\,df_e(\exp B)$, which agrees with $df_e$ only at $B = 0$. The \emph{corrected chart} undoes this factor, applying $(T_B^{-1})^*$ to the gradient and $T_B^{-1}$ to the Adam step; it agrees with the multiplicative update to first order and differs only in the retraction. We evaluate $\varphi(-\mathrm{ad}_B)$ by a block matrix exponential. Both chart updates break down where $T_B$ is singular.

\paragraph{Subgroups of the conformal group}
The ten bivector generators of $\mathrm{CGA}(3) = \mathrm{Cl}(4,1)$, written in the null basis $e_o = \frac12(e_+ + e_-)$, $e_\infty = e_- - e_+$, split into three rotations $e_{ij}$, three translations $e_i\wedge e_\infty$, three transversions $e_i\wedge e_o$, and the dilation $e_+\wedge e_-$. The four groups of Fig.~\ref{fig:lie_subalgebras} are coordinate subspaces of this basis: translations $T(3)$ (dimension 3), rigid motions $SE(3)$ (rotations and translations, dimension 6), similarities $\mathrm{Sim}(3)$ (rigid motions and dilation, dimension 7), and the full conformal group (dimension 10). The source cloud $P$ has 29 points along three perpendicular arms of lengths $1.5$, $1.0$ and $0.5$, so that chirality and per-axis scale are visible. A single teacher motor $M^\star = \exp(B^\star)$ is drawn from the full conformal algebra (seed 13; coefficient scales $0.5$ for translations, $0.3$ for rotations, $0.25$ for the dilation and $0.12$ for transversions), and the target is $Y = \mathrm{down}\big(M^\star\,\mathrm{up}(P)\,\widetilde{M^\star}\big) + \sigma\epsilon$, where $\mathrm{up}(x) = e_o + x + \frac12|x|^2e_\infty$. The loss is the sum of squared residuals over the $29\times3 = 87$ coordinates. All four groups start from the same masked initialization and are optimized in one vectorized run of 500 multiplicative Adam steps at learning rate $0.1$, which takes about 4\,s on an Apple M3 laptop CPU. As references we run BFGS on each group's active coordinates and, for the three restricted groups, the closed-form optima (centroid alignment, Procrustes, and Umeyama).

Table~\ref{tab:cga_subgroups} lists the final losses. The three restricted groups reach their closed-form optima, and the conformal group recovers the teacher: noise-free, the final motor matches $M^\star$ to a relative error of $6\cdot10^{-12}$, while the static chart stalls at a loss of $8.1\cdot10^{-4}$ and recovers the teacher coefficients only to $1.6\cdot10^{-3}$. With noise $\sigma = 0.05$, all methods reach the same losses, and the conformal loss sits near the expected floor of $(87-10)\sigma^2 \approx 0.19$. The difference between the updates is starker away from the identity: started from a rotation of 4\,rad, the descent path crosses $g = -1$, a conjugate point of $\exp$ at which every chart is singular. The multiplicative update passes through and converges (to $-M^\star$, the same transformation) with loss $9\cdot10^{-21}$, whereas the static chart stalls at $0.41$ and the corrected chart diverges (Fig.~\ref{fig:lie_update_rule}). Finally, the renormalization trick $g\leftarrow g(g\tilde g)^{-1/2}$, which works in $\mathrm{Cl}(1,3)$ (see below), is not a retraction in $\mathrm{Cl}(4,1)$: there $g\tilde g$ has a grade-4 part, the iterate leaves the group ($|g\tilde g - 1|\sim10^{14}$), and the loss stalls near $10^2$.

\begin{table}[htbp]
\centering
\caption{Final losses for the four subgroups of the conformal group of Fig.~\ref{fig:lie_subalgebras}, fitted to one shared conformal target, noise-free (top) and with noise $\sigma = 0.05$ (bottom).}
\label{tab:cga_subgroups}
\small
\begin{tabular}{@{}llccccc@{}}
\toprule
& group & dim & multiplicative (ours) & static chart & BFGS & closed form \\
\midrule
$\sigma = 0$ & $T(3)$ & 3 & 2.7797 & 2.7797 & 2.7797 & 2.7797 \\
& $SE(3)$ & 6 & 1.5897 & 1.5897 & 1.5897 & 1.5897 \\
& $\mathrm{Sim}(3)$ & 7 & 0.39614 & 0.39614 & 0.39614 & 0.39614 \\
& conformal & 10 & $1.7\cdot10^{-21}$ & $8.1\cdot10^{-4}$ & $2.5\cdot10^{-21}$ & -- \\
\midrule
$\sigma = 0.05$ & $T(3)$ & 3 & 2.9138 & 2.9138 & 2.9138 & 2.9138 \\
& $SE(3)$ & 6 & 1.6620 & 1.6620 & 1.6620 & 1.6620 \\
& $\mathrm{Sim}(3)$ & 7 & 0.56218 & 0.56218 & 0.56218 & 0.56218 \\
& conformal & 10 & 0.21071 & 0.21268 & 0.21071 & -- \\
\bottomrule
\end{tabular}
\end{table}

\begin{figure}[htbp]
\centering
\includegraphics[width=0.98\textwidth]{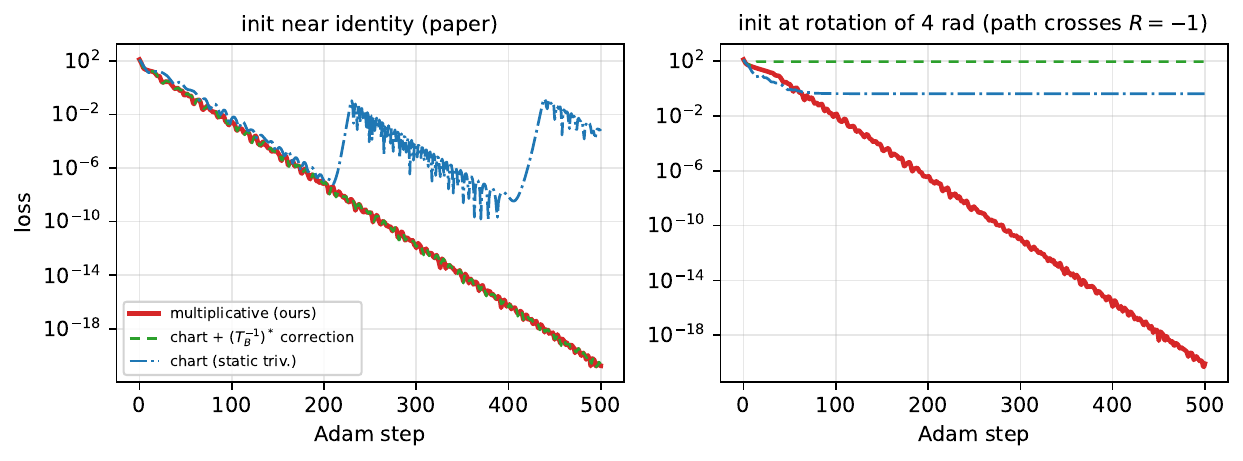}
\caption{Training loss on the full conformal group under the multiplicative update (ours), the static chart, and the corrected chart. \emph{Left:} from the default initialization near the identity. \emph{Right:} from an initialization at a rotation of 4\,rad, where the descent path crosses the conjugate point $g = -1$ of $\exp$.}
\label{fig:lie_update_rule}
\end{figure}

\paragraph{Lorentz frame recovery}
We work in $\mathrm{Cl}(1,3)$ with $e_0$ timelike, as in Appendix~\ref{sec:em_details}, and $\mathfrak{g} = \mathfrak{spin}(1,3)$ spanned by the three boosts $e_0\wedge e_i$ and three rotations $e_i\wedge e_j$. A static unit charge at the origin produces the Coulomb field $F_k$ at $K = 48$ events on the shell $0.8\le|r|\le2$. The unknown ship frame is the rotor $R^\star$ combining a boost of generator rapidity $0.8$ along $(1, 0.5, 0.2)$ with a rotation of $0.7$\,rad about $(0.3, -1, 0.6)$, and the ship measures $\tilde F_k = R^\star F_k\widetilde{R^\star} + \epsilon_k$, with Gaussian noise at $5\%$ of the RMS signal ($\sigma = 1.83\cdot10^{-2}$ per bivector coefficient). We minimize $\mathcal{L}(R) = \frac{1}{6K}\sum_k\lVert\tilde F_k - RF_k\tilde R\rVert^2$, whose noise floor is $\sigma^2$. We compare the multiplicative update against the static and corrected charts, BFGS on the chart coordinates, and three parameterizations one would write without a Lorentz-group library: (i) the eight even-grade coefficients of $R$, renormalized after every step by $R\leftarrow R(R\tilde R)^{-1/2}$ (in $\mathrm{Cl}(1,3)$, $R\tilde R = s + pI$ with $I^2 = -1$, so this is a complex inverse square root that must be re-derived for each algebra); (ii) a $4\times4$ matrix $\Lambda$ with the penalty $\lambda\lVert\Lambda^\top\eta\Lambda - \eta\rVert_F^2$, with $\eta$ the Minkowski metric and $\lambda\in\{0.1, 1, 10\}$; and (iii) the same matrix projected onto its polar factor after every step, the standard retraction of compact-manifold libraries, which projects onto $O(4)$. Every method runs 400 Adam steps from the same initialization, with the learning rate chosen from $\{0.3, 0.1, 0.03, 0.01\}$ as the smallest one within $5\%$ of the best final loss. We read the observer's rapidity off the recovered rotor from the image of the time axis; the ground truth is $0.7957$ (speed $0.662c$), which differs from $2|B^\star_{\mathrm{boost}}| = 0.8$ by the Baker--Campbell--Hausdorff cross terms between the non-parallel boost and rotation. The whole comparison, including the learning-rate sweeps, takes 20\,s on an Apple M3 laptop CPU.

Table~\ref{tab:lorentz} reports the results. Every on-group method reaches the noise floor, and on this task the three trivialized updates agree to all printed digits. The penalized matrix also reaches the floor but leaves the group (constraint violation $1.9\cdot10^{-3}$). The polar projection stalls at $48\sigma^2$ with violation $1.7$, because $O(4)$ contains no boosts.

\begin{table}[htbp]
\centering
\caption{Lorentz frame recovery from 48 noisy Faraday bivectors in $\mathrm{Cl}(1,3)$ (400 Adam steps; noise floor $\sigma^2 = 3.36\cdot10^{-4}$). Violation is $\lVert\Lambda^\top\eta\Lambda - \eta\rVert_F$ for the vector representation $\Lambda$ of the recovered transformation; $\Lambda$ error is $\min_\pm\lVert\Lambda\mp\Lambda^\star\rVert_F/\lVert\Lambda^\star\rVert_F$; rapidity error is against the ground truth $0.7957$.}
\label{tab:lorentz}
\small
\begin{tabular}{@{}lcccc@{}}
\toprule
method & loss$/\sigma^2$ & violation & $\Lambda$ error & rapidity error \\
\midrule
multiplicative (ours) & 1.03 & $7.3\cdot10^{-16}$ & $3.9\cdot10^{-3}$ & $1.4\cdot10^{-3}$ \\
static chart & 1.03 & $1.0\cdot10^{-15}$ & $3.9\cdot10^{-3}$ & $1.4\cdot10^{-3}$ \\
corrected chart & 1.03 & $6.1\cdot10^{-16}$ & $3.9\cdot10^{-3}$ & $1.4\cdot10^{-3}$ \\
chart, BFGS (40 steps) & 1.03 & -- & $3.9\cdot10^{-3}$ & $1.4\cdot10^{-3}$ \\
even multivector + renormalize & 1.06 & $2.0\cdot10^{-15}$ & $1.0\cdot10^{-2}$ & $1.1\cdot10^{-2}$ \\
matrix + penalty ($\lambda = 0.1$) & 1.03 & $1.9\cdot10^{-3}$ & $3.9\cdot10^{-3}$ & $1.7\cdot10^{-3}$ \\
matrix + $O(4)$ polar projection & 47.8 & $1.7$ & $0.64$ & $0.23$ \\
\bottomrule
\end{tabular}
\end{table}

\end{document}